\pdfoutput=1

\documentclass[11pt]{article}
\RequirePackage[OT1]{fontenc} 
\RequirePackage[numbers]{natbib}
\RequirePackage[colorlinks,citecolor=blue,urlcolor=blue]{hyperref}

\usepackage[utf8]{inputenc} 
\usepackage[T1]{fontenc}    
\usepackage[colorlinks]{hyperref} 

\usepackage{hyperref}
\usepackage{url}

\usepackage[utf8]{inputenc} 
\usepackage[T1]{fontenc}    
\usepackage{hyperref}       
\usepackage{url}            
\usepackage{booktabs}       
\usepackage{amsfonts}       
\usepackage{nicefrac}       
\usepackage{microtype}      
\usepackage{xcolor}         
\usepackage{algorithm}
\usepackage{algorithmic}
\usepackage{kotex}
\usepackage{enumitem}

\usepackage{subcaption}
\usepackage{amsmath}
\usepackage{amssymb}
\usepackage{mathtools}
\usepackage{amsthm}
\usepackage{multirow}

\usepackage{amsmath,amsfonts,bm}

\def\eqref#1{equation~\ref{#1}}

\def\1{\bm{1}}

\DeclareMathAlphabet{\mathsfit}{\encodingdefault}{\sfdefault}{m}{sl}
\SetMathAlphabet{\mathsfit}{bold}{\encodingdefault}{\sfdefault}{bx}{n}

\hypersetup{colorlinks,citecolor={teal}}

\definecolor{coral}{rgb}{0.99, 0.41, 0.47}

\newcommand{\err}[0]{\mathrm{err}}
\newcommand{\diag}[0]{\mathrm{diag}}

\theoremstyle{plain}
\newtheorem{theorem}{Theorem}[section]
\newtheorem{proposition}[theorem]{Proposition}

\newtheorem{corollary}[theorem]{Corollary}
\newtheorem{remark}[theorem]{Remark}
\theoremstyle{definition}

\usepackage{setspace}
\usepackage[left=2cm,right=2cm,top=4cm,bottom=4cm,a4paper]{geometry}

\begin{document}
	\title{Scale-invariant Bayesian Neural Networks with Connectivity Tangent Kernel}
	\date{}
	\author{\parbox{\linewidth}{\centering
		SungYub Kim$^{1}$, Sihwan Park$^{1}$, Kyungsu Kim$^{2,3}$\thanks{Co-corresponding author}, Eunho Yang$^{1,4}$\footnotemark[1]\\ 
		{\small $^{1}$Korea Advanced Institute of Science and Technology (KAIST)}\\
		{\small $^{2}$Medical AI Research Center, Research Institute for Future Medicine, Samsung Medical Center, Seoul, Korea}\\
		{\small $^{3}$Department of Data Convergence and Future Medicine, Sungkyunkwan University School of Medicine, Seoul, Korea}\\ 
        {\small $^{4}$AITRICS, Seoul, Korea}\\
		{\footnotesize E-mail: sungyub.kim@kaist.ac.kr, sihwan.park@mli.kaist.ac.kr, kskim.doc@gmail.com, eunhoy@kaist.ac.kr} 
		}
	}

	\maketitle 
	
\begin{abstract} 
Explaining generalizations and preventing over-confident predictions are central goals of studies on the loss landscape of neural networks. Flatness, defined as loss invariability on perturbations of a pre-trained solution, is widely accepted as a predictor of generalization in this context. However, the problem that flatness and generalization bounds can be changed arbitrarily according to the scale of a parameter was pointed out, and previous studies partially solved the problem with restrictions: Counter-intuitively, their generalization bounds were still variant for the function-preserving parameter scaling transformation or limited only to an impractical network structure. As a more fundamental solution, we propose new prior and posterior distributions invariant to scaling transformations by \textit{decomposing} the scale and connectivity of parameters, thereby allowing the resulting generalization bound to describe the generalizability of a broad class of networks with the more practical class of transformations such as weight decay with batch normalization. We also show that the above issue adversely affects the uncertainty calibration of Laplace approximation and propose a solution using our invariant posterior. We empirically demonstrate our posterior provides effective flatness and calibration measures with low complexity in such a practical parameter transformation case, supporting its practical effectiveness in line with our rationale.
\end{abstract}
  
\section{Introduction}\label{sec:intro}

Though neural networks (NNs) have experienced extraordinary success, understanding the generalizability of NNs and successfully using them in real-world contexts still faces a number of obstacles \citep{kendall2017what, ovadia2019can}. It is a well-known enigma, for instance, why such NNs generalize well and do not suffer from overfitting \citep{neyshabur2014search,zhang2017understanding,arora2018stronger}. Recent research on the loss landscape of NNs seeks to reduce these obstacles. \citet{hochreiter1995flat} proposed a theory known as flat minima (FM): the flatness of local minima (i.e., loss invariability w.r.t. parameter perturbations) is positively correlated with network generalizability, as empirically demonstrated by \citet{jiang2020fantastic}. Concerning overconfidence, \citet{mackay1992laplace} suggested an approximated Bayesian posterior using the curvature information of local minima, and \citet{daxberger2021laplace} underlined its practical utility.

Nonetheless, the limitations of the FM hypothesis were pointed out by \citet{dinh2017sharp, li2018visualizing}. By rescaling two successive layers, \citet{dinh2017sharp} demonstrated it was possible to modify a flatness measure without modifying the functions, hence allowing extraneous variability to be captured in the computation of generalizability. Meanwhile, \citet{li2018visualizing} argued that weight decayregularization~\citep{loshchilov2018decoupled} is an important limitation of the FM hypothesis as it leads to a contradiction of the FM hypothesis in practice; the weight decay sharpens the pre-trained solutions of NNs by downscaling the parameters, whereas the weight decay actually improves the generalization performance of NNs in general cases \citep{zhang2018three}. In short, they suggest that scaling transformation on network parameters (e.g., re-scaling layers and weight decay) may lead to a contradiction of the FM hypothesis.

To resolve this contradiction, we investigate PAC-Bayesian prior and posterior distributions to derive a new \textbf{scale-invariant generalization bound}. Unlike related works \citep{tsuzuku2020nfm, kwon2021asam}, our bound guarantees the invariance for a general class of function-preserving parameter scaling transformation with a broad class of networks \citep{petzka2021relative} (Secs. \ref{Sec:sub_dist} and \ref{Sec:sub_bound}).

This bound is derived from the scale invariance of the prior and poster distributions, which guarantees not only the scale invariance of the bound but also its substance the Kullback-Leibler~(KL) divergence-based kernel; we named this new term with scale-invariance property as empirical Connectivity Tangent Kernel (CTK) as it can be considered as a modification of empirical Neural Tangent Kernel~\citep{jacot2018ntk}. Consequently, we define a novel sharpness metric named \textbf{Connectivity Sharpness (CS)} as a trace of CTK. Empirically, we verify our CS has a better prediction for generalization performance of NNs than existing sharpness measures \citep{liang2019fisher, keskar2017large, neyshabur2017exploring} with a low-complexity (Sec. \ref{Sec:sub_CS}), with confirming its stronger correlation to generalization error (Sec. \ref{subsec:gen-metric}). 

We also found the contradiction of the FM hypothesis turns into meaningless predictive uncertainty amplifying issues in the Bayesian NN regime (Sec. \ref{subsec:pitfall}), and can alleviate this issue by using Bayesian NN based on the posterior distribution of our PAC-Bayesian analysis. We call the resulting Bayesian NN as \textbf{Connectivity Laplace (CL)}, as it can be seen as a variation of Laplace Approximation (LA; \citet{mackay1992laplace}) using different Jacobian. In particular, we provide pitfalls of weight decay regularization with BN in LA and its remedy using our posterior (Sec. \ref{subsec:pitfall}) to suggest practical utilities of our Bayesian NNs (Sec. \ref{subsec:exp-CTKGP}). We summarize our contributions as follows:

\begin{itemize}[leftmargin=*]
\item Unlike related studies, our resulting (PAC-Bayes) generalization bound guarantees the invariance for a general class of function-preserving parameter scaling transformation with a broad class of networks (Sec. \ref{Sec:sub_dist} and \ref{Sec:sub_bound}). Based on this novel PAC-Bayes bound, we propose a low-complexity sharpness measure (Sec. \ref{Sec:sub_CS}).
\item We provide pitfalls of weight decay regularization with BN in LA and its remedy using our posterior (Sec. \ref{subsec:pitfall}).
\item We empirically confirm the strong correlation between generalization error and our sharpness metric (Sec. \ref{subsec:gen-metric}) and visualize pitfalls of weight decay with LA in synthetic data and practical utilities of our Bayesian NNs (Sec. \ref{subsec:exp-CTKGP}).
\end{itemize}
        
\section{PAC-Bayes bound with scale-invariance}\label{sec:pac-bayes}

\subsection{Background}
    
\textbf{Setup and Definitions} We consider a Neural Network (NN), $f(\cdot, \cdot): \mathbb{R}^{D} \times \mathbb{R}^{P} \rightarrow \mathbb{R}^{K}$, given input $x \in \mathbb{R}^{D}$ and network parameter $\theta \in \mathbb{R}^{P}$. Hereafter, we consider one dimensional vector {a one-dimensional vector} as a single column matrix unless otherwise stated. We use the output of NN $f(x, \theta)$ as a prediction for input $x$. Let $\mathcal{S} := \{(x_n, y_n)\}_{n=1}^{N}$ denote the independently and identically distributed (i.i.d.) training data drawn from true data distribution $\mathcal{D}$, where $x_n \in \mathbb{R}^{D}$ and $y_n \in \mathbb{R}^{K}$ are input and output representation of $n$-th training instance, respectively. For simplicity, we denote concatenated input and output of all instances as $\mathcal{X} := \{x:(x, y) \in \mathcal{S}\}$ and $\mathcal{Y} := \{y:(x, y) \in \mathcal{S}\}$, respectively and $f(\mathcal{X},\theta) \in \mathbb{R}^{NK}$ as a concatenation of $\{f(x_n,\theta)\}_{n=1}^{N}$. Given a prior distribution of network parameter $p(\theta)$ and a likelihood function $p(\mathcal{S}|\theta):= \prod_{n=1}^{N}p(y_n|x_n, \theta) :=\prod_{n=1}^{N}p(y_n|f(x_n, \theta))$, Bayesian inference defines posterior distribution of network parameter $\theta$ as $p(\theta | \mathcal{S}) = \frac{1}{Z(\mathcal{S})} \exp(-\mathcal{L}(\mathcal{S}, \theta)) := \frac{1}{Z(\mathcal{S})} p(\theta) p(\mathcal{S}|\theta), Z(\mathcal{S}):=\int p(\theta) p(\mathcal{S}|\theta) d\theta$ where $\mathcal{L}(\mathcal{S}, \theta):=-\log p(\theta)-\sum_{n=1}^{N}\log p(y_n|x_n, \theta)$ is training loss and $Z(\mathcal{S})$ is normalizing factor. For example, the likelihood function for regression task will be Gaussian: $p(y|x,\theta) = \mathcal{N}(y|f(x, \theta), \sigma^2 \mathbf{I}_{k})$ where $\sigma$ is (homoscedastic) observation noise scale. For classification task, we treat it as a one-hot regression task following \citet{lee2019wide} and \citet{he2020ntkgp}. While we applied this modification for theoretical tractability, \citet{lee2020finite, hui2021evaluation} showed this modification offers reasonable performance competitive to the cross-entropy loss. Details on this treatment is given in Appendix~\ref{sec:classification}.

\textbf{Laplace Approximation} In general, the exact computation for the Bayesian posterior of a network parameter is intractable. The Laplace Approximation (LA; \cite{mackay1992laplace}) proposes to approximate the posterior distribution with a Gaussian distribution defined as $p_\mathrm{LA}(\psi|\mathcal{S}) \sim \mathcal{N}(\psi| \theta^*, (\nabla^{2}_{\theta}\mathcal{L}(\mathcal{S}, \theta^*))^{-1})$ where $\theta^*\in \mathbb{R}^{P}$ is a pre-trained parameter with training loss and $\nabla^{2}_{\theta}\mathcal{L}(\mathcal{S},\theta^*) \in \mathbb{R}^{P \times P}$ is Hessian of loss function w.r.t. parameter at $\theta^*$. 

Recent works on LA replace the Hessian matrix with (Generalized) Gauss-Newton matrix to make computation tractable \citep{khan2019dnn2gp, immer2021local}. With this approximation, the LA posterior of regression problem can be represented as:
\begin{align}
    p_\mathrm{LA}(\psi|\mathcal{S}) \sim \mathcal{N}(\psi| \theta^*, (\underbrace{\mathbf{I}_{P} / \alpha^2}_\textrm{Damping} + \underbrace{\mathbf{J}_{\theta}^{\top} \mathbf{J}_{\theta}/\sigma^2 }_\textrm{Curvature})^{-1}) \label{eq:laplace-fisher}
\end{align}
where $\alpha, \sigma > 0$ and $\mathbf{I}_{P} \in \mathbb{R}^{P \times P}$ is a identity matrix and $\mathbf{J}_{\theta} \in \mathbb{R}^{NK \times P}$ is a concatenation of $\mathbf{J}_{\theta}(x, \theta^*) \in \mathbb{R}^{K \times P}$ (Jacobian of $f$ w.r.t. $\theta$ at input $x$ and parameter $\theta^*$) along training input $\mathcal{X}$.
Since covariance of \eqref{eq:laplace-fisher} is inverse of $P \times P$  matrix, further sub-curvature approximation was considered including diagonal, Kronecker-factored approximate curvature (KFAC), last-layer, and sub-network \citep{ritter2018a, kristiadi2020being, daxberger2021subnet}. Furthermore, they found that proper selection of prior scale $\alpha$ is needed to balance the dilemma between {overconfidence and underfitting} in LA.

\textbf{PAC-Bayes bound with data-dependent prior} We consider a PAC-Bayes generalization error bound of classification task used in \citet{mcallester1999pac, maria2021tighter} (especially, equation (7) of \citet{maria2021tighter}). Let $\mathbb{P}$ be any PAC-Bayes prior distribution over $\mathbb{R}^{P}$ independent of training dataset $\mathcal{S}$ and $\err (\cdot, \cdot): \mathbb{R}^{K \times K} \rightarrow [0,1]$ be a error function which is defined separately from the loss function. For any constant $\delta \in (0,1]$ and $\lambda >0$, and any PAC-Bayes posterior distribution $\mathbb{Q}$ over $\mathbb{R}^{P}$, the following holds with probability at least $1-\delta$: $\err_{\mathcal{D}}(\mathbb{Q}) \le \err_{\mathcal{S}}(\mathbb{Q}) + \sqrt{\frac{\mathrm{KL}[\mathbb{Q}\|\mathbb{P}] + \log(2\sqrt{N}/\delta)}{2N}}$ where $\err_{\mathcal{D}}(\mathbb{Q}) := \mathbb{E}_{(x,y) \sim \mathcal{D}, \theta \sim \mathbb{Q}} [\err(f(x,\theta), y) ]$, $\err_{\mathcal{S}}(\mathbb{Q}) := \mathbb{E}_{(x,y) \sim \mathcal{S}, \theta \sim \mathbb{Q}} [\err(f(x,\theta), y) ]$, and $N$ denotes the cardinality of $\mathcal{S}$.
That is, $\err_{\mathcal{D}}(\mathbb{Q})$ and $\err_{\mathcal{S}}(\mathbb{Q})$ are generalization error and empirical error, respectively. The only restriction on $\mathbb{P}$ here is that it \textit{cannot depend} on the dataset $S$.

Following the recent discussion in  \citet{maria2021tighter}, one can construct data-dependent PAC-Bayes bounds by (\romannumeral 1) randomly partitioning dataset $\mathcal{S}$ into $\mathcal{S}_\mathbb{Q}$ and $\mathcal{S}_\mathbb{P}$ so that they are independent, (\romannumeral 2) using a PAC-Bayes prior distribution $\mathbb{P}_\mathcal{D}$ only dependent of $\mathcal{S}_\mathbb{P}$ (i.e., independent of $\mathcal{S}_\mathbb{Q}$ so $\mathbb{P}_\mathcal{D}$ belongs to $\mathbb{P}$),  (\romannumeral 3) using a PAC-Bayes posterior distribution $\mathbb{Q}$ dependent of entire dataset $\mathcal{S}$, and (\romannumeral 4) computing empirical error $\err_{\mathcal{S}_\mathbb{Q}}(\mathbb{Q})$ with target subset $\mathcal{S}_\mathbb{Q}$ (not entire dataset $\mathcal{S}$). In summary, one can modify the aforementioned original PAC-Bayes bound through our data-dependent prior $\mathbb{P}_{c}$ as 
\begin{align}  
    \err_{\mathcal{D}}(\mathbb{Q}) 
    &\le \err_{\mathcal{S}_\mathbb{Q}}(\mathbb{Q}) + \sqrt{\frac{\mathrm{KL}[\mathbb{Q}\|\mathbb{P}_\mathcal{D}] + \log(2\sqrt{N_\mathbb{Q}}/\delta)}{2N_\mathbb{Q}}} \label{eq:pac-bayes-dep}
\end{align}
where $N_\mathbb{Q}$ is the cardinality of $\mathcal{S}_\mathbb{Q}$. We denote sets of input and output of splitted datasets ($\mathcal{S}_\mathbb{P}, \mathcal{S}_\mathbb{Q}$) as $\mathcal{X}_\mathbb{P}, \mathcal{Y}_\mathbb{P}, \mathcal{X}_\mathbb{Q}, \mathcal{Y}_\mathbb{Q}$ for simplicity.

\subsection{Design of PAC-Bayes prior and posterior}\label{Sec:sub_dist}

Our goal is to construct scale-invariant $\mathbb{P}_{\mathcal{D}}$ and $\mathbb{Q}$. To this end, we first assume a pre-trained parameter $\theta^* \in \mathbb{R}^{P}$ of the negative log-likelihood function with $\mathcal{S}_\mathbb{P}$. This parameter can be attained with standard NN optimization procedures (e.g., stochastic gradient descent (SGD) with momentum). Then, we consider linearized NN at the pre-trained parameter with the auxiliary variable $c \in \mathbb{R}^{P}$ as
\begin{align}\label{eq:linearization-ctk}
    g^\mathrm{lin}_{\theta^*}(x, c) := f(x,\theta^*) + \mathbf{J}_{\theta}(x,\theta^*)\diag(\theta^*) c
\end{align}
where $\diag$ is a vector-to-matrix diagonal operator. {Note that \eqref{eq:linearization-ctk} is the first-order Taylor approximation (i.e., linearization) of NN with perturbation $\theta^* \odot c$ given input $x$ and parameter $\theta^*$: $g^\mathrm{pert}_{\theta^*}(x,c) := f(x, \theta^* + \theta^* \odot c) = f(x, \theta^* + \diag(\theta^*)c) \approx g^\mathrm{lin}_{\theta^*}(x, c)$, where $\odot$ denotes element-wise multiplication~(Hadamard product) of two vectors. Here we write the perturbation in parameter space as $\theta^* \odot c$ instead of single variable such as $\delta \in \mathbb{R}^{P}$. This design of linearization matches the scale of perturbation (i.e., $\diag(\theta^*)\odot c$) to the scale of $\theta^*$ in a component-wise manner.} Similar idea was proposed in the context of pruning at initialization \citep{lee2018snip, lee2019signal} to measure the importance of each connection independently of its weight. In this context, our perturbation can be viewed as \textbf{perturbation in connectivity space by decomposing the scale and connectivity of parameter}.

Based on this, we define a data-dependent prior ($\mathbb{P}_{\mathcal{D}}$) over connectivity as
\begin{align}
    \mathbb{P}_{\theta^*}(c) &:= \mathcal{N}(c\,|\,\mathbf{0}_{P}, \alpha^2 \mathbf{I}_{p})\label{eq:ctkgp-c-prior}.
\end{align}
This distribution can be translated to a distribution over parameter by considering the distribution of perturbed parameter ($\psi := \theta^* + \theta^* \odot c$): $\mathbb{P}_{\theta^*}(\psi) := \mathcal{N}(\psi \,|\, \theta^*, \alpha^2 \mathrm{diag}(\theta^*)^2 )$. We now define the PAC-Bayes posterior over connectivity $\mathbb{Q}(c)$ as follows:
\begin{align}
\mathbb{Q}_{\theta^*}(c) &:= \mathcal{N}(c| \mu_{\mathbb{Q}}, \Sigma_{\mathbb{Q}})\label{eq:ctkgp-c-posterior}\\
\mu_{\mathbb{Q}} &:= \frac{\Sigma_{\mathbb{Q}}\mathbf{J}_{c}^\top \left( \mathcal{Y} - f(\mathcal{X},\theta^*)\right)}{\sigma^2} = \frac{\Sigma_{\mathbb{Q}}\diag(\theta^*)\mathbf{J}_{\theta}^{\top}\left( \mathcal{Y} - f(\mathcal{X},\theta^*)\right)}{\sigma^2}\label{eq:ctkgp-c-posterior-mean}\\
\Sigma_{\mathbb{Q}} &:= \left( \frac{\mathbf{I}_{P}}{\alpha^2} + \frac{\mathbf{J}_{c}^\top \mathbf{J}_{c}}{\sigma^2} \right)^{-1} = \left( \frac{\mathbf{I}_{P}}{\alpha^2} + \frac{\diag(\theta^*)\mathbf{J}_{\theta}^\top \mathbf{J}_{\theta} \diag(\theta^*)}{\sigma^2} \right)^{-1}\label{eq:ctkgp-c-posterior-cov}
\end{align}
where $\mathbf{J}_{c} \in NK \times P$ is a concatenation of $\mathbf{J}_{c}(x, \mathbf{0}_{P}) := \mathbf{J}_{\theta}(x,\theta^*)\diag(\theta^*) \in \mathbb{R}^{K \times P}$ (i.e., Jacobian of perturbed NN $g^\mathrm{pert}_{\theta^*}(x,c)$ w.r.t. $c$ at input $x$ and connectivity $\mathbf{0}_{P}$) along training input $\mathcal{X}$. Our PAC-Bayes posterior $\mathbb{Q}_{\theta^*}$ is \textit{the posterior of Bayesian linear regression problem w.r.t. connectivity} $c$ : $y_i = f(x_i, \theta^*) + \mathbf{J}_{\theta}(x_i,\theta^*)\diag(\theta^*) c + \epsilon_i$ where $(x_i, y_i) \in \mathcal{S}$ and $\epsilon_i$ are i.i.d. samples of $\mathcal{N}(\epsilon_i | \mathbf{0}_{K}, \sigma^2 \mathbf{I}_{K})$. Again, it is equivalent to the posterior distribution over parameter $\mathbb{Q}_{\theta^*}(\psi) = \mathcal{N}\left(\psi|\theta^* + \theta^* \odot \mu_\mathbb{Q}, ( \diag(\theta^*)^{-2}/\alpha^2 + \mathbf{J}_{\theta}^\top \mathbf{J}_{\theta}/\sigma^2 \right)^{-1} )$ where $\diag (\theta^*)^{-2} := (\diag(\theta^*)^{-1})^{2}$ by assuming that all components of $\theta^*$ are non-zero. Note that this assumption can be easily satisfied by considering the prior and posterior distribution of non-zero components of NNs only. Although we choose this restriction for theoretical tractability, future work can modify this choice to achieve diverse predictions by considering the distribution of zero coordinates. We refer to Appendix \ref{sec:pac-bayes-posterior-derivation} for detailed derivations of $\mathbb{Q}_{\theta^*}(c)$ and $\mathbb{Q}_{\theta^*}(\psi)$.

\begin{remark}[Two-phase training]
\label{remark:two-phase-training}
    The prior distribution in \eqref{eq:ctkgp-c-prior} is data-dependent priors since they depend on the pre-trained parameter $\theta^*$ optimized on training dataset $\mathcal{S}_\mathbb{P}$. On the other hand, posterior distribution in \eqref{eq:ctkgp-c-posterior} depend on both $\mathcal{S}_\mathbb{P}$ (through $\theta^*$) and $\mathcal{S}_\mathbb{Q}$ (through Bayesian linear regression). Intuitively, one attain the PAC-Bayes posterior $\mathbb{Q}$ with two-phase training: pre-train with $\mathcal{S}_\mathbb{P}$ and Bayesian fine-tuning with $\mathcal{S}_\mathbb{Q}$. A similar idea of linearized fine-tuning was proposed in the context of transfer learning in \citet{achille2021lqf, maddox2021fast}.
\end{remark}

Now we provide an invariance property of prior and posterior distributions w.r.t. function-preserving scale transformations as follows: The main intuition behind this proposition is that \textbf{Jacobian w.r.t. connectivity is invariant to the function-preserving scaling transformation}, i.e., \\ $\mathbf{J}_{\theta}(x,\mathcal{T}(\theta^*))\diag (\mathcal{T}(\theta^*)) = \mathbf{J}_{\theta}(x,\theta^*)\diag (\theta^*)$.
\begin{proposition}[Scale-invariance of PAC-Bayes prior and posterior]\label{prop:distribution-invariance}
Let $\mathcal{T}: \mathbb{R}^{P} \rightarrow \mathbb{R}^{P}$ is a invertible diagonal linear transformation such that $f(x,\mathcal{T}(\psi)) = f(x, \psi)\text{ , }\forall x\in \mathbb{R}^{D}, \forall \psi \in \mathbb{R}^{P}$. Then, both PAC-Bayes prior and posterior are invariant under $\mathcal{T}$:
\begin{align*}
    \mathbb{P}_{\mathcal{T}(\theta^*)}(c) \stackrel{d}{=} \mathbb{P}_{\theta^*}(c),\quad \mathbb{Q}_{\mathcal{T}(\theta^*)}(c) \stackrel{d}{=} \mathbb{Q}_{\theta^*}(c).
\end{align*}
Furthermore, generalization and empirical errors are also invariant to $\mathcal{T}$.
\end{proposition}

\subsection{Resulting PAC-Bayes bound}\label{Sec:sub_bound}

Now we plug in prior and posterior into the modified PAC-Bayes generalization error bound in \eqref{eq:pac-bayes-dep}. Accordingly, we obtain a novel generalization error bound, named \textbf{PAC-Bayes-CTK}, which is guaranteed to be invariant to scale transformations (hence without the contradiction of FM hypothesis mentioned in Sec. \ref{sec:intro}).
\begin{theorem}[PAC-Bayes-CTK and its invariance] \label{thm:pac-bayes-ctk}
    Let us assume pre-trained parameter $\theta^*$ with data $\mathcal{S}_\mathbb{P}$. By applying $\mathbb{P}_{\theta^*}, \mathbb{Q}_{\theta^*}$ to data-dependent PAC-Bayes bound (\eqref{eq:pac-bayes-dep}), we get
    \begin{align}\label{eq:pac-bayes-connectivity}
        \err_{\mathcal{D}}(\mathbb{Q}_{\theta^*}) 
        &\le \err_{\mathcal{S}_\mathbb{Q}}(\mathbb{Q}_{\theta^*}) + \sqrt{
        \overbrace{\underbrace{\frac{\mu_{\mathbb{Q}}^{\top}\mu_{\mathbb{Q}}}{4\alpha^2 N_\mathbb{Q}}}_\textrm{(average)   perturbation} + \underbrace{\sum_{i=1}^{P}\frac{h\left(\beta_i \right)}{{4N_\mathbb{Q}}}}_\textrm{sharpness}}^\textrm{KL divergence} + \frac{\log(2\sqrt{N_\mathbb{Q}}/\delta)}{2N_\mathbb{Q}}}
    \end{align}
    where $\{\beta_i\}_{i=1}^{P}$ are eigenvalues of $(\mathbf{I}_{P} + \frac{\alpha^2}{\sigma^2}\mathbf{J}_{c}^\top \mathbf{J}_c)^{-1}$ and $h(x) := x-\log(x)-1$. \textbf{This upper bound is invariant to $\mathcal{T}$} for the function-preserving scale transformation in Proposition \ref{prop:distribution-invariance}. 
\end{theorem}

Note that recent works on solving FM contradiction only focused on the scale-invariance of sharpness metric: Indeed, their generalization bounds are not invariant to scale transformations due to the scale-dependent terms (equation (34) in \citet{tsuzuku2020nfm} and equation (6) in \citet{kwon2021asam}). On the other hand, generalization bound in \citet{petzka2021relative} (Theorem 11 in their paper) only holds for single-layer NNs, whereas ours has no restrictions for network structure. Therefore, to the best of our knowledge, our PAC-Bayes bound is the first scale-invariant PAC-Bayes bound. To highlight our theoretical implications, we show the representative cases of $\mathcal{T}$ in Proposition \ref{prop:distribution-invariance} in Appendix \ref{sec:scaling-transformations} {(e.g., weight decay for network with BN}), where the generalization bounds of the other studies are variant, but ours is invariant, resolving the FM contradiction on bound level.

The following corollary explains why we name PAC-Bayes bound in Theorem \ref{thm:pac-bayes-ctk} PAC-Bayes-CTK.
\begin{corollary}[Relation between CTK and PAC-Bayes-CTK]\label{coro:ctk-eigenvalues}
    Let us define empirical \textbf{Connectivity Tangent Kernel}~(CTK) of $\mathcal{S}$ as $\mathbf{C}_\mathcal{X}^{\theta^*}:= \mathbf{J}_{c} \mathbf{J}_{c}^\top = \mathbf{J}_{\theta} \diag (\theta^*)^2 \mathbf{J}_{\theta}^\top  \in \mathbb{R}^{N K \times N K}$ by removing below term?
    Note that empirical CTK has $T (\le NK)$ non-zero eigenvalues of $\{\lambda_i\}_{i=1}^{T}$, then following holds for $\{ \beta \}_{i=1}^{P}$ in Theorem \ref{thm:pac-bayes-ctk}: (\romannumeral 1) $\beta_i = \sigma^2/(\sigma^2 + \alpha^2 \lambda_i) < 1$ for $i=1,\dots,T$ and (\romannumeral 2) $\beta_i = 1$ for $i=T+1,\dots,P$. Since $h(1)=0$, this means $P-T$ terms of summation in sharpness part of PAC-Bayes-CTK vanishes to 0. Furthermore, this sharpness term of PAC-Bayes-CTK is a monotonically increasing function for each eigenvalue of empirical CTK. 
\end{corollary}
Note that Corollary \ref{coro:ctk-eigenvalues} clarifies why $\sum_{i=1}^{P} h(\beta_i) / 4N_\mathbb{Q}$ in Theorem \ref{thm:pac-bayes-ctk} is called {the} sharpness term of PAC-Bayes-CTK. As eigenvalues of CTK measures {the} sensitivity of output w.r.t. perturbation on connectivity, a sharp pre-trained parameter would have large CTK eigenvalues, so increasing the sharpness term and the generalization gap by according to Corollary \ref{coro:ctk-eigenvalues}.

Finally, Proposition \ref{prop:ctk-invariance} shows that empirical CTK is also scale-invariant.
\begin{proposition}[Scale-invariance of empirical CTK]\label{prop:ctk-invariance}
Let $\mathcal{T}: \mathbb{R}^{P} \rightarrow \mathbb{R}^{P}$ be an function-preserving scale transformation in Proposition \ref{prop:distribution-invariance}. Then empirical CTK at parameter $\psi$ is invariant under $\mathcal{T}$:
\begin{align}
    \mathbf{C}^{\mathcal{T}(\psi)}_{xy} 
    &:= \mathbf{J}_{\theta}(x,\mathcal{T}(\psi))\mathrm{diag}(\mathcal{T}(\psi)^2) \mathbf{J}_{\theta}(y,\mathcal{T}(\psi))^{\top}= \mathbf{C}^\psi_{xy} 
    \text{ , }\forall x,y\in \mathbb{R}^{D}, \forall \psi \in \mathbb{R}^{P}. \label{ctk_invariant_main}
\end{align}
\end{proposition}
\begin{remark}[Connections to empirical NTK]\label{rmk:connections-to-ntk}
The empirical CTK $\mathbf{C}^{\psi}_{xy}$ resembles the existing empirical Neural Tangent Kernel~(NTK) at parameter $\psi$ \citep{jacot2018ntk}: $\Theta_{xy}^{\psi}:= \mathbf{J}_{\theta}(x,\psi) \mathbf{J}_{\theta}(y,\psi)^{\top} \in \mathbb{R}^{k \times k}$. Note that the deterministic NTK in \citet{jacot2018ntk} is the infinite-width limiting kernel at initialized parameters, while {empirical NTK} can be defined on any parameter of {a} finite-width NN. The main difference between empirical CTK and the existing empirical NTK is in the definition of Jacobian. In CTK, Jacobian is computed w.r.t. connectivity $c$ while the empirical NTK uses Jacobian w.r.t. parameters $\theta$. Therefore, another PAC-Bayes bound can be derived from the linearization of $f^\mathrm{pert}_{\theta^*}(x,\delta) := f(x,\theta^* + \delta) \approx f^\mathrm{lin}_{\theta^*}(x,\delta)$ where $f^\mathrm{lin}_{\theta^*}(x,\delta) := f(x,\theta^*) + \mathbf{J}_{\theta}(x,\theta^*)\delta$. As this bound is related to the eigenvalues of $\Theta^{\theta^*}_\mathcal{X}$, we call this bound \textbf{PAC-Bayes-NTK} and provide derivations in Appendix \ref{sec:proofs}. Note PAC-Bayes-NTK is not scale-invariant as $\Theta^{\mathcal{T}(\psi)}_{xy} \neq \Theta^{\psi}_{xy}$ in general.    
\end{remark}

\subsection{Computing approximate bound in real world problems}\label{sec:computing-bound}

To verify that PAC-Bayes bound in Theorem \ref{thm:pac-bayes-ctk} is non-vacuous, we compute it for real-world problems. We use CIFAR-10 and 100 datasets \citep{krizhevsky2009learning}, where the 50K training instances are randomly partitioned into $\mathcal{S}_\mathbb{P}$ of cardinality 45K and $\mathcal{S}_\mathbb{Q}$ of cardinality 5K. We pre-train ResNet-18 \citep{he2016deep} with a mini-batch size of 1K  on $\mathcal{S}_\mathbb{P}$ with SGD of initial learning rate 0.4 and momentum 0.9. We use cosine annealing for learning rate scheduling \citep{loshchilov2016sgdr} with a warmup for the initial 10\% training step. We fix $\delta=0.1$ and select $\alpha, \sigma$ based on the negative log-likelihood of $\mathcal{S}_\mathbb{Q}$.

To compute the \eqref{eq:pac-bayes-connectivity}, one need (\romannumeral 1) $\mu_\mathbb{Q}$-based perturbation term, (\romannumeral 2) $\mathbf{C}^{\theta^*}_\mathcal{X}$-based sharpness term, and (\romannumeral 3) samples from PAC-Bayes posterior $\mathbb{Q}_{\theta^*}$. $\mu_\mathbb{Q}$ in \eqref{eq:ctkgp-c-posterior-mean} can be obtained by minimizing $\arg\min_{c \in \mathbb{R}^{P}}L(c)=\frac{1}{2N}\|\mathcal{Y}-f(\mathcal{X},\theta^*)-\mathbf{J}_{c}c\|^2 + \frac{\sigma^2}{2\alpha^2N}c^\top c$ by first-order optimality condition. Note that this problem is a convex optimization problem w.r.t. $c$, since $c$ is the parameter of the linear regression problem. We use Adam optimizer \citep{kingma2014adam} with fixed learning rate 1e-4 to solve this. For sharpness term, we apply the Lanczos algorithm to approximate the eigenspectrum of $\mathbf{C}^{\theta^*}_\mathcal{X}$ following \cite{ghorbani2019investigation}. We use 100 Lanczos iterations based on the their setting. Lastly, we estimate empirical error and test error with 8 samples of CL/LL implementation of Randomize-Then-Optimize (RTO) framework \citep{bardsley2014randomize, matthews2017sample}. We refer to Appendix \ref{sec:kfac-cl} for the RTO implementation of CL/LL.

\begin{table}[t]
\centering
\caption{\footnotesize Results for experiments on PAC-Bayes-CTK estimation.}
\label{tab:bound_ctk}
\resizebox{\textwidth}{!}{%
\begin{tabular}{@{}c|cccc|cccc@{}}
\toprule
                & \multicolumn{4}{c|}{CIFAR-10}         & \multicolumn{4}{c}{CIFAR-100}         \\ \midrule
Parameter scale & 0.5     & 1.0     & 2.0     & 4.0     & 0.5     & 1.0     & 2.0     & 4.0     \\ \toprule
$\mathrm{tr}(\mathbf{C}^{\theta^*}_{\mathcal{X}})$ & 14515.0039      & 14517.7793      & 14517.3506      & 14518.4746      & 12872.6895      & 12874.4395      & 12873.9512      & 12875.541       \\
Bias            & 13.9791 & 13.4685 & 13.3559 & 13.3122 & 25.3686 & 24.8064 & 24.9102 & 24.7557 \\
Sharpness       & 24.6874 & 24.6938 & 24.6926 & 24.6941 & 26.0857 & 26.0894 & 26.0874 & 26.0916 \\
KL         & 19.3332 & 19.0812 & 19.0243 & 19.0032 & 25.7271 & 25.4479 & 25.4988 & 25.4236 \\ 
Test err.    & 0.0468 ± 0.0014 & 0.0463 ± 0.0013 & 0.0462 ± 0.0012 & 0.0460 ± 0.0013 & 0.2257 ± 0.0020 & 0.2252 ± 0.0017 & 0.2256 ± 0.0015 & 0.2253 ± 0.0017 \\ \midrule
PAC-Bayes-CTK       & 0.0918 ± 0.0013 & 0.0911 ± 0.0011 & 0.0909 ± 0.0011 & 0.0907 ± 0.0009 & 0.2874 ± 0.0034 & 0.2862 ± 0.0031 & 0.2860 ± 0.0032 & 0.2862 ± 0.0032 \\ \bottomrule
\end{tabular}
} 
\end{table}

Table \ref{tab:bound_ctk} provides the bound and related term of the resulting model. First of all, we found that our estimated PAC-Bayes-CTK is non-vacuous \citep{zhou2018non}: estimated bound is better than guessing at random. Note that the non-vacuous bound is not trivial in PAC-Bayes analysis: only a few PAC-Bayes literatures \citep{dziguaite2017pac, zhou2018non, maria2021tighter} verified the non-vacuous property of their bound, and other PAC-Bayes literatures \citep{foret2020sharpness, tsuzuku2020nfm} do not. To check the invariance property of our bound, we scale the scale-invariant parameters in ResNet-18 (i.e., parameters preceding BN layers) for fixed constants $\{0.5, 1.0, 2.0, 4.0\}$. Note that this scaling does not change the function represented by NN due to BN layers, and the error bound should be preserved. Table \ref{tab:bound_ctk} shows that our bound and related terms are invariant to these transformations. On the other hand, PAC-Bayes-NTK is very sensitive to parameter scale, as shown in Table \ref{tab:bound_ntk} in Appendix \ref{sec:lanczos}.

\subsection{Connectivity Sharpness and its efficient computation}\label{Sec:sub_CS} 
Now, we focus on the fact that the trace of CTK is also invariant to the parameter scale by Proposition \ref{prop:ctk-invariance}. Unlike PAC-Bayes-CTK/NTK, a trace of CTK/NTK does not require onerous hyper-parameter selection of $\delta, \alpha, \sigma$. Therefore, we simply define $\textbf{CS}(\theta^*) := \mathrm{tr}(\mathbf{C}^{\theta^*}_{\mathcal{X}})$ as a practical sharpness measure at $\theta^*$, named \textbf{Connectivity Sharpness}~(CS) to detour the complex computation of PAC-Bayes-CTK. This metric can be easily applied to find NNs with better generalization, similar to other sharpness metrics (e.g., trace of Hessian), as shown in \cite{jiang2020fantastic}. We evaluate the detecting performance of CS in Sec. \ref{subsec:gen-metric}. The following corollary shows how CS can explain the generalization performance of NNs, conceptually.
\begin{corollary}[Connectivity sharpness, Informal]
\label{coro:connectivity-sharpness}
Let us assume CTK and KL divergence term of PAC-Bayes-CTK as defined in Theorem \ref{thm:pac-bayes-ctk}. Then, if CS vanishes to zero or infinity, the KL divergence term of PAC-Bayes-CTK also does so.
\end{corollary}

As the trace of a matrix can be efficiently estimated by Hutchinson's method \citep{hutchinson1989stochastic}, one can compute the CS \textit{without explicitly computing the entire CTK}. We refer to Appendix \ref{sec:hutchison} for detailed procedures of computing CS. As CS is invariant to function-preserving scale transformations by Theorem~\ref{prop:ctk-invariance}, it also does not contradict the FM hypothesis. 

\section{Bayesian NNs with scale-invariance}\label{sec:bnn}

This section provides the practical implications of the posterior distribution used in PAC-Bayes analysis. We interpret the PAC-Bayes posterior $\mathbb{Q}_{\theta^*}$ in \eqref{eq:ctkgp-c-posterior} as a modified result of LA \citep{mackay1992laplace}. Then, we show this modification improves existing LA in the presence of weight decay regularization. Finally, we explain how one can efficiently construct a Bayesian NN from \eqref{eq:ctkgp-c-posterior}. 

\subsection{Pitfalls of weight decay with BN in Laplace Approximation}\label{subsec:pitfall}

One can view parameter space version of $\mathbb{Q}_{\theta^*}$ as a modified version of $p_\mathrm{LA}(\psi|\mathcal{S})$ in \eqref{eq:laplace-fisher} by (\romannumeral 1) replacing isotropic damping term to the parameter scale-dependent damping term ($\diag(\theta^*)^{-2}$) and (\romannumeral 2) adding perturbation $\theta^* \odot \mu_\mathbb{Q}$ to the mean of Gaussian distribution. In this section, we focus on the effect of replacing the damping term of LA in the presence of weight decay of batch normalized NNs. We refer to \cite{antoran2021linearised, antoran2022adapting} for the discussion on the effect of adding perturbation to the LA with linearized NNs.

The main difference between covariance term of LA \eqref{eq:laplace-fisher} and \eqref{eq:ctkgp-c-posterior-cov} is the definition of Jacobian (i.e. parameter or connectivity) similar to the difference between empirical CTK and NTK in remark \ref{rmk:connections-to-ntk}. Therefore, we name $p_\mathrm{CL}(\psi|\mathcal{S}) \sim \mathcal{N}(\psi|\theta^*, \left( \diag(\theta^*)^{-2}/\alpha^2 + \mathbf{J}_{\theta}^\top \mathbf{J}_{\theta}/\sigma^2 \right)^{-1} )$ as \textbf{Connectivity Laplace (CL)} approximated posterior.

To compare CL posterior and existing LA, we explain how weight decay regularization with BN produces unexpected side effects in existing LA. This side effect can be quantified if we consider linearized NN for LA, called Linearized Laplace (LL; \citet{foong2019between}). Note that LL is well known to be better calibrated than non-linearized LA for estimating 'in-between' uncertainty. By assuming $\sigma^2 \ll \alpha^2$, the predictive distribution of linearized NN for \eqref{eq:laplace-fisher} and CL are 
\begin{align}
    f^\mathrm{lin}_{\theta^*}(x,\psi)&|p_\mathrm{LA}(\psi|\mathcal{S})
    \sim \mathcal{N}(f(x,\theta^*), \alpha^2 \Theta_{xx}^{\theta^*} - \alpha^2 \Theta_{x\mathcal{X}}^{\theta^*}\Theta_{\mathcal{X}}^{\theta^*-1}\Theta_{\mathcal{X}x}^{\theta^*})\label{eq:pred-unc-ll-posterior}\\
    f^\mathrm{lin}_{\theta^*}(x,\psi)&|p_\mathrm{CL}(\psi|\mathcal{S})
    \sim \mathcal{N}(f(x,\theta^*), \alpha^2 \mathbf{C}_{xx}^{\theta^*} -\alpha^2  \mathbf{C}_{x\mathcal{X}}^{\theta^*}\mathbf{C}_{\mathcal{X}}^{\theta^*-1}\mathbf{C}_{\mathcal{X}x}^{\theta^*})\label{eq:pred-unc-cl-posterior}.
\end{align}
for any input $x \in \mathbb{R}^{d}$ where $\mathcal{X}$ in subscript of CTK/NTK means concatenation.  We refer to Appendix \ref{sec:derivation-pred-unc} for the detailed derivations. The following proposition shows how weight decay regularization on scale-invariant parameters introduced by BN can amplify the predictive uncertainty of \eqref{eq:pred-unc-ll-posterior}. Note that the primal regularizing effect of weight decay originates through regularization on scale-invariant parameters \citep{van2017l2, zhang2018three}.

\begin{proposition}[Uncertainty amplifying effect for LL]\label{prop:uncertainty-amplifying}
Let us assume that $\mathcal{W}_{\gamma}: \mathbb{R}^{P} \rightarrow \mathbb{R}^{P}$ is a weight decay regularization on scale-invariant parameters (e.g., parameters preceding BN layers) by multiplying $\gamma < 1$ and all the non-scale-invariant parameters are fixed. Then, predictive uncertainty of LL is amplified by $1/\gamma^2 > 1$ while predictive uncertainty of CTK is preserved:
\begin{align*}
    &\textrm{Var}_{\psi \sim p_\mathrm{LA}(\psi|\mathcal{S})}(f^\mathrm{lin}_{\mathcal{W}_{\gamma}(\theta^*)}(x,\psi)) = \textrm{Var}_{\psi \sim p_\mathrm{LA}(\psi|\mathcal{S})}(f^\mathrm{lin}_{\theta^*}(x,\psi))/\gamma^2\\
    &\textrm{Var}_{\psi \sim p_\mathrm{CL}(\psi|\mathcal{S})}(f^\mathrm{lin}_{\mathcal{W}_{\gamma}(\theta^*)}(x,\psi)) = \textrm{Var}_{\psi \sim p_\mathrm{CL}(\psi|\mathcal{S})}(f^\mathrm{lin}_{\theta^*}(x,\psi))
\end{align*}
where $\textrm{Var}(\cdot)$ is variance of random variable.
\end{proposition}

Recently, \cite{antoran2021linearised, antoran2022adapting} observed similar pitfalls in Proposition \ref{prop:uncertainty-amplifying}. However, their solution requires a more complicated hyper-parameter search: \textit{independent prior selection for each normalized parameter group}. On the other hand, CL does not increase the hyper-parameter to be optimized compared to LL. We believe this difference will make CL more attractive to practitioners.
 
\section{Experiments}

Here we describe experiments that demonstrate (\romannumeral 1) the effectiveness of Connectivity Sharpness (CS) as a generalization measurement metric and (\romannumeral 2) the usefulness of Connectivity Laplace (CL) as an efficient general-purpose Bayesian NN: {CL resolves the contradiction of FM hypothesis and shows stable calibration performance to the selection of prior scale.}

\subsection{Connectivity Sharpness as a generalization measurement metric} \label{subsec:gen-metric}

To verify that the CS actually has a better correlation with generalization performance compared to existing sharpness measures, we evaluate the three correlation metrics on the CIFAR-10 dataset: (a) Kendall's rank-correlation coefficient ($\tau$) \citep{kendall1938new} (b) granulated Kendall's coefficient and their average ($\Psi$) \citep{jiang2020fantastic} (c) conditional independence test ($\mathcal{K}$) \citep{jiang2020fantastic}. {For all correlation metrics, a larger value means a stronger correlation between sharpness and generalization.}

We compare CS to following baseline sharpness measures: trace of Hessian ($\mathrm{tr}(\mathbf{H})$;  \citet{keskar2017large}), trace of empirical Fisher ($\mathrm{tr}(\mathbf{F})$; \citet{jastrzebski2021catastrophic}), trace of empirical NTK at $\theta^*$ ($\mathrm{tr}(\mathbf{\Theta^{\theta^*}})$), Fisher-Rao (FR ;\citet{liang2019fisher}) metric, Adaptive Sharpness (AS; \citet{kwon2021asam}), and four PAC-Bayes bound based measures: Sharpness-Orig. (SO), Pacbayes-Orig. (PO), Sharpness-Mag. (SM), and Pacbayes-Mag. (PM), which are eq. (52), (49), (62), (61) in \citet{jiang2020fantastic}. For computing granulated Kendall's correlation, we use 5 hyper-parameters (network depth, network width, learning rate, weight decay, and mini-batch size) and 3 options for each (thus we train models with $3^5=243$ different training configurations). We vary the depth and width of NN based on VGG-13 \citep{Simonyan15}. We refer to Appendix \ref{supp:exp-corr} for experimental details.

\begin{table}[t]
\centering 
\caption{\footnotesize Correlation analysis of sharpness measures with generalization gap. We refer Sec. \ref{subsec:gen-metric} for the details of sharpness measures (row) and correlation metrics for sharpness-generalization relationship (column).}
\label{table:sharpness}
\resizebox{0.8\textwidth}{!}{%
\begin{tabular}{@{}cccccccccc|c@{}}
\toprule
 & tr($\mathbf{H}$) & tr($\mathbf{F}$) & tr($\mathbf{\Theta}^{\theta^*}$) & SO & PO & SM & PM & AS & FR & CS \\ \midrule
$\tau$ (rank corr.) & 0.706 & 0.679 & 0.703 & 0.490 & 0.436 & 0.473 & 0.636 & 0.755 & 0.649 & \textbf{0.837} \\\midrule
network depth & 0.764 & 0.652 & \textbf{0.978} & -0.358 & -0.719 & 0.774 & 0.545 & 0.756 & 0.771 & \textbf{0.978} \\
network width & 0.687 & 0.922 & 0.330 & -0.533 & -0.575 & 0.495 & 0.564 & 0.827 & 0.921 & \textbf{0.978} \\
mini-batch size & 0.976 & 0.810 & \textbf{0.988} & 0.859 & 0.893 & 0.909 & 0.750 & 0.829 & 0.685 & 0.905 \\
learning rate & 0.966 & 0.713 & \textbf{1.000} & 0.829 & 0.874 & 0.057 & 0.621 & 0.794 & 0.565 & 0.897 \\
weight decay & -0.031 & -0.103 & 0.402 & 0.647 & 0.711 & 0.168 & 0.211 & 0.710 & 0.373 & \textbf{0.742} \\\midrule
$\Psi$ (avg.) & 0.672 & 0.599 & 0.739 & 0.289 & 0.237 & 0.481 & 0.538 & 0.783 & 0.663 & \textbf{0.900} \\\midrule
$\mathcal{K}$ (cond. MI) & 0.320 & 0.243 & 0.352 & 0.039 & 0.041 & 0.049 & 0.376 & 0.483 & 0.288 & \textbf{0.539} \\ \bottomrule
\end{tabular}%
} 
\end{table}

{In Table~\ref{table:sharpness}, CS shows the best results for $\tau$, $\Psi$, and $\mathcal{K}$ compared to all other sharpness measures. Also, granulated Kendall of CS is higher than other sharpness measures for 3 out of 5 hyper-parameters and competitive to other sharpness measures with the leftover hyper-parameters. The main difference of CS with other sharpness measures is in the correlation with \textit{network width} and \textit{weight decay}: For \textit{network width}, we found that sharpness measures except CS, $\mathrm{tr}(\mathbf{F})$, AS and FR fail to capture strong correlation. While SO/PO can capture the correlation with \textit{weight decay}, we believe it is due to the weight norm term of SO/PO. However, this term would interfere in capturing the sharpness-generalization correlation related to the number of parameters (i.e., width/depth), while CS/AS does not suffer from such a problem. Also, it is notable that FR fails to capture this correlation despite its invariant property.}

\subsection{Connectivity Laplace as an efficient general-purpose Bayesian NN}\label{subsec:exp-CTKGP}

To assess the effectiveness of CL as a general-purpose Bayesian NN, we consider uncertainty calibration on UCI dataset and CIFAR-10/100.

\begin{table}[t]
\centering
\caption{\footnotesize Test negative log-likelihood on two UCI variants~\citep{hernandez2015probabilistic, foong2019between}}
\label{tab:uci}
\resizebox{\textwidth}{!}{%
\begin{tabular}{@{}ccccc|cccc@{}}
\toprule
 & \multicolumn{4}{c|}{Original \citep{hernandez2015probabilistic} } & \multicolumn{4}{c}{GAP variants \citep{foong2019between} } \\ \midrule
 & Deep Ensemble & MCDO & LL & CL & Deep Ensemble & MCDO & LL & CL \\ \midrule
boston\_housing & 2.90 ± 0.03 & 2.63 ± 0.01 & \textbf{2.85 ± 0.01} & 2.88 ± 0.02 & 2.71 ± 0.01 & 2.68 ± 0.01 & \textbf{2.74 ± 0.01} & 2.75 ± 0.01 \\
concrete\_strength & 3.06 ± 0.01 & 3.20 ± 0.00 & 3.22 ± 0.01 & \textbf{3.11 ± 0.02} & 4.03 ± 0.07 & 3.42 ± 0.00 & \textbf{3.47 ± 0.01} & 4.03 ± 0.02 \\
energy\_efficiency & 0.74 ± 0.01 & 1.92 ± 0.01 & 2.12 ± 0.01 & \textbf{0.83 ± 0.01} & 0.77 ± 0.01 & 1.78 ± 0.01 & 2.02 ± 0.01 & 0.90 ± 0.02 \\
kin8nm & -1.07 ± 0.00 & -0.80 ± 0.01 & -0.90 ± 0.00 & \textbf{-1.07 ± 0.00} & -0.94 ± 0.00 & -0.71 ± 0.00 & -0.87 ± 0.00 & \textbf{-0.93 ± 0.00} \\
naval\_propulsion & -4.83 ± 0.00 & -3.85 ± 0.00 & -4.57 ± 0.00 & \textbf{-4.76 ± 0.00} & -2.22 ± 0.33 & -3.36 ± 0.01 & -3.66 ± 0.11 & \textbf{-3.80 ± 0.07} \\
power\_plant & 2.81 ± 0.00 & 2.91 ± 0.00 & 2.91 ± 0.00 & \textbf{2.81 ± 0.00} & 2.91 ± 0.00 & 2.97 ± 0.00 & 2.98 ± 0.00 & \textbf{2.91 ± 0.00} \\
protein\_structure & 2.89 ± 0.00 & 2.96 ± 0.00 & 2.91 ± 0.00 & \textbf{2.89 ± 0.00} & 3.11 ± 0.00 & 3.07 ± 0.00 & \textbf{3.07 ± 0.00} & 3.13 ± 0.00 \\
wine & 1.21 ± 0.00 & 0.96 ± 0.01 & \textbf{1.24 ± 0.01} & 1.27 ± 0.01 & 1.48 ± 0.01 & 1.03 ± 0.00 & 1.45 ± 0.01 & \textbf{1.43 ± 0.00} \\
yacht\_hydrodynamics & 1.26 ± 0.04 & 2.17 ± 0.06 & \textbf{1.20 ± 0.04} & 1.25 ± 0.04 & 1.71 ± 0.03 & 3.06 ± 0.02 & 1.78 ± 0.02 & \textbf{1.74 ± 0.01} \\ \bottomrule
\end{tabular}%
} 
\end{table}

\textbf{UCI regression datasets}\label{sec:uci} We implement full-curvature versions of LL and CL and evaluate these to the 9 UCI regression datasets~\citep{hernandez2015probabilistic} and its GAP-variants~\citep{foong2019between} to compare calibration performance on \textit{in-between} uncertainty. We train MLP with a single hidden layer. We fix $\sigma=1$ and choose $\alpha$ from \{0.01, 0.1, 1, 10, 100\} using log-likelihood of validation dataset. We use 8 random seeds to compute the average and standard error of the test negative log-likelihoods. The following two tables show test NLL for LL/CL and 2 baselines (Deep Ensemble \citep{lakshminarayanan2017deep} and Monte-Carlo DropOut (MCDO; \citet{gal2016dropout})). Eight ensemble members are used in Deep Ensemble, and 32 MC samples are used in LL, CL, and MCDO. Table~\ref{tab:uci} show that CL performs better than LL for 6 out of 9 datasets. Although LL shows better calibration results for 3 datasets in both settings, we would like to point out that the performance gaps between LL and CL were not severe as in the other 6 datasets, where CL performs better. 

\textbf{Image Classification} We evaluate the uncertainty calibration performance of CL on CIFAR-10/100. As baseline methods, we consider Deterministic network, Monte-Carlo Dropout (MCDO; \citep{gal2016dropout}), Monte-Carlo Batch Normalization (MCBN; \citep{teye2018bayesian}), and Deep Ensemble~\citep{lakshminarayanan2017deep}, Batch Ensemble \citep{wen2020batchensemble}, LL \citep{khan2019dnn2gp, daxberger2021laplace}. We use Randomize-Then-Optimize (RTO) implementation of LL/CL in Appendix \ref{sec:kfac-cl}. We measure Expected Calibration Error~(ECE; \citet{guo2017ece}), negative log-likelihood~(NLL), and Brier score~(Brier.) for ensemble predictions. We also measure the area under receiver operating curve~(AUC) for OOD detection, where we set the SVHN~\citep{netzer2011svhn} dataset as an OOD dataset. For more details on the experimental setting, please refer to Appendix~\ref{supp:exp-bnn}.

Table~\ref{table:resnet-18-cifar-100} shows uncertainty calibration results on CIFAR-100. We refer to Appendix~\ref{supp:exp-bnn} for results on other settings, including CIFAR-10 and VGGNet \citep{Simonyan15}. Our CL shows better results than baselines for all uncertainty calibration metrics~(NLL, ECE, Brier., and AUC) except Deep Ensemble. This means scale-invariance of CTK improves Bayesian inference, which is consistent with the results in toy examples. Although the Deep Ensemble presents the best results in 3 out of 4 metrics, it requires full training from initialization for each ensemble member, while LL/CL requires only a post-hoc training upon the pre-trained NN for each member. Particularly noteworthy is that CL presents competitive results with Deep Ensemble, even with much smaller computations.

\begin{table*}[t]
\caption{\footnotesize Uncertainty calibration results on CIFAR-100~\citep{krizhevsky2009learning} for ResNet-18~\citep{he2016deep}}\label{table:resnet-18-cifar-100}
\centering
\resizebox{\textwidth}{!}{%
\begin{tabular}{c|cccc}
\toprule
& \multicolumn{4}{c}{CIFAR-100}           \\ 
\midrule
&NLL~($\downarrow$)&ECE~($\downarrow$)&Brier.~($\downarrow$)&AUC~($\uparrow$) \\
\midrule
Deterministic& 1.5370 ± 0.0117 & 0.1115 ± 0.0017 & 0.3889 ± 0.0031 & -\\
MCDO & 1.4264 ± 0.0110 & 0.0651 ± 0.0008 & 0.3925 ± 0.0020 & 0.6907 ± 0.0121 \\
MCBN & 1.4689 ± 0.0106 & 0.0998 ± 0.0016 & 0.3750 ± 0.0028 & 0.7982 ± 0.0210 \\
Batch Ensemble & 1.4029 ± 0.0031 & 0.0842 ± 0.0005 & 0.3582 ± 0.0010 & 0.7887 ± 0.0115 \\
Deep Ensemble & 1.0110 & 0.0507 & 0.2740 & 0.7802 \\\midrule
Linearized Laplace & 1.1673 ± 0.0093 & 0.0632 ± 0.0010 & 0.3597 ± 0.0020 &  0.8066 ± 0.0120 \\
Connectivity Laplace (Ours) & 1.1307 ± 0.0042 & 0.0524 ± 0.0009 & 0.3319 ± 0.0005 & 0.8423 ± 0.0204 \\
\bottomrule 
\end{tabular}
}
\end{table*}

\begin{figure*}[!t]
    \begin{subfigure}[t]{0.33\linewidth}
        \includegraphics[width=\linewidth]{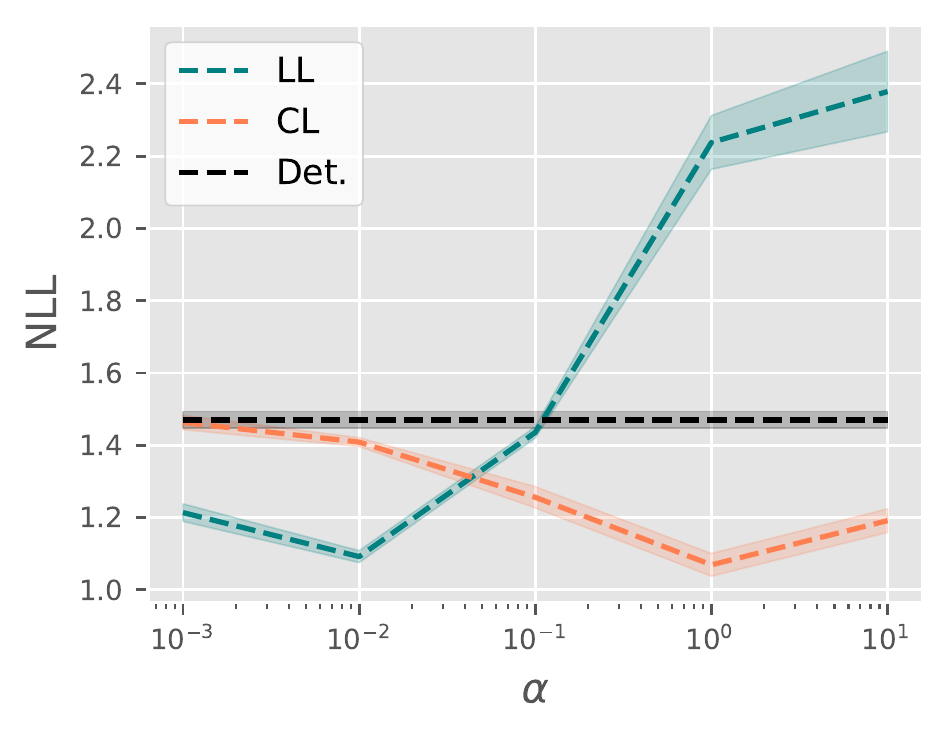}
        \captionsetup{justification=centering}
        \caption{NLL}
        \label{fig:step_10}
    \end{subfigure}
    \begin{subfigure}[t]{0.33\linewidth}
        \includegraphics[width=\linewidth]{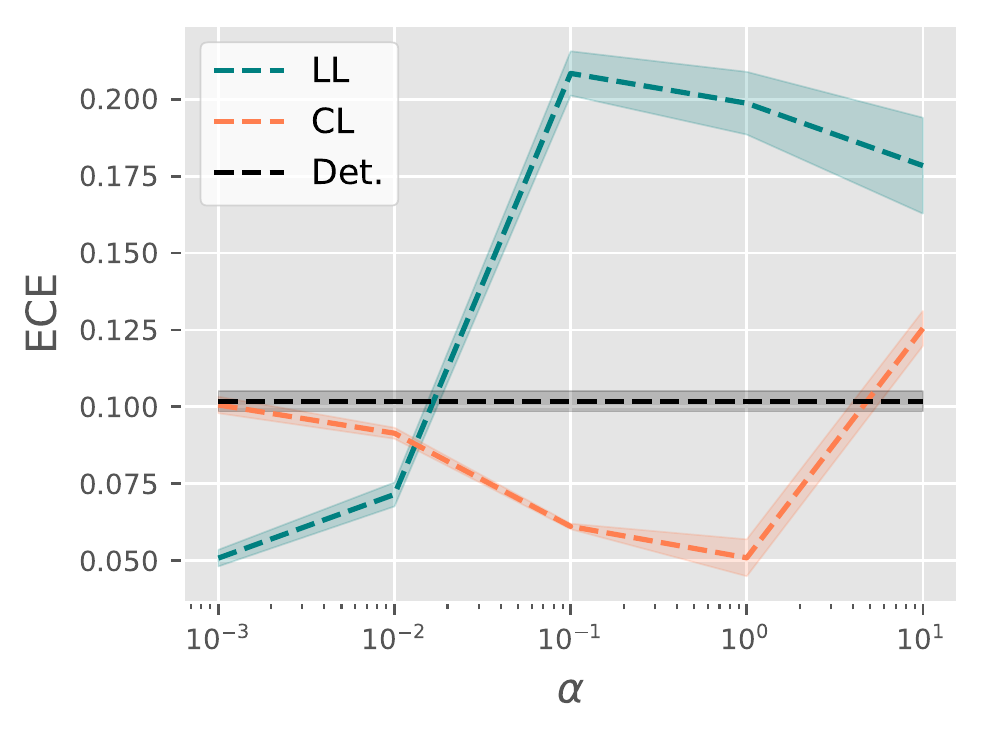}
        \captionsetup{justification=centering}
        \caption{ECE}
        \label{fig:step_0}
    \end{subfigure}
    \begin{subfigure}[t]{0.33\linewidth}
        \includegraphics[width=\linewidth]{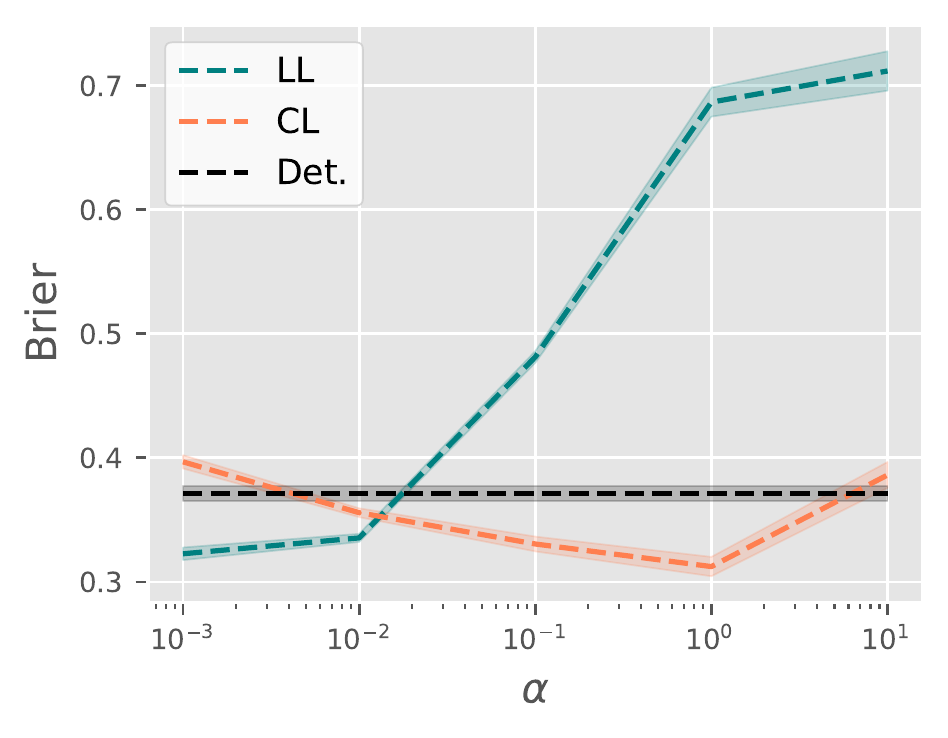}
        \captionsetup{justification=centering}
        \caption{Brier Score}
        \label{fig:step_100}
    \end{subfigure}
    \caption{\footnotesize Sensitivity to $\alpha$. Expected calibration error (ECE), Negative Log-likelihood (NLL), and Brier score results on corrupted CIFAR-100 for ResNet-18. Showing the mean (line) and standard deviation (shaded area) across four different seeds.}
    \label{fig:alpha}
    \vspace{-0.5cm}
\end{figure*}

\textbf{Robustness to the selection of prior scale}
{Figure~\ref{fig:alpha} shows the uncertainty calibration (i.e. NLL, ECE, Brier) results over various $\alpha$ values for LL, CL, and Deterministic (Det.) baseline. As mentioned in previous works \citep{ritter2018a, kristiadi2020being}, the uncertainty calibration results of LL is extremely sensitive to the selection of $\alpha$. Especially, LL shows severe under-fitting for large $\alpha$ (i.e. small damping) regime. On the other hand, CL shows stable performance in the various ranges of $\alpha$.}
    
\section{Conclusion}
This study introduced novel PAC-Bayes prior and posterior distributions to extend the robustness of generalization bound w.r.t. parameter transformation by \textit{decomposing} the scale and connectivity of parameters. The resulting generalization bound is guaranteed to be invariant of any function-preserving scale transformations. This result successfully solved the problem that the contradiction of the FM hypothesis caused by the general scale transformation could not be solved in the existing generalization error bound, thereby allowing the theory to be much closer to reality. As a result of the theoretical enhancement, our posterior distribution for PAC-analysis can also be interpreted as an improved Laplace Approximation without pathological failures in weight decay regularization. Therefore, we expect this fact contributes to reducing the theory-practice gap in understanding the generalization effect of NN, leading to follow-up studies that interpret this effect more clearly.




\clearpage 
\bibliographystyle{unsrtnat}
\bibliography{iclr2023_conference.bib}
\clearpage 


\newpage
\appendix

\section{Proofs}\label{sec:proofs}

\subsection{Proof of Proposition \ref{prop:distribution-invariance}}
    
\begin{proof}
Since the prior $\mathbb{P}_{\theta^*}(c)$ is independent to the parameter scale, $\mathbb{P}_{\theta^*}(c)\stackrel{d}{=}\mathbb{P}_{\mathcal{T}(\theta^*)}(c)$ is trivial.
For Jacobian w.r.t. parameters, we have
\begin{align*}
    \mathbf{J}_{\theta}(x, \mathcal{T}(\psi)) = \frac{\partial }{\partial \mathcal{T}(\psi)} f(x, \mathcal{T}(\psi)) = \frac{\partial }{\partial \mathcal{T}(\psi)} f(x, \psi) = \mathbf{J}_{\theta}(x, \psi) \mathcal{T}^{-1}
\end{align*}
Then, the Jacobian of NN w.r.t. connectivity at $\mathcal{T}(\psi)$ holds
\begin{align}
    \mathbf{J}_{\theta}(x, \mathcal{T}(\psi))\diag (\mathcal{T}(\psi))
    &= \mathbf{J}_{\theta}(x, \psi) \mathcal{T}^{-1} \mathcal{T} \diag (\psi)
    \\&= \mathbf{J}_{\theta}(x, \psi)\diag (\psi)\label{eq:jacobian-c-invariance}
\end{align}
where the first equality holds from the above one and the fact that $\mathcal{T}$ is a diagonal linear transformation. Therefore, the covariance of posterior is invariant to $\mathcal{T}$.
\begin{align*}
    &\left( \frac{\mathbf{I}_{P}}{\alpha^2} + \frac{\diag(\mathcal{T}(\theta^*))\mathbf{J}_{\theta}^\top(\mathcal{X},\mathcal{T}(\theta^*)) \mathbf{J}_{\theta}(\mathcal{X},\mathcal{T}(\theta^*)) \diag(\mathcal{T}(\theta^*))}{\sigma^2} \right)^{-1} 
    \\&= \left( \frac{\mathbf{I}_{P}}{\alpha^2} + \frac{\diag(\theta^*)\mathbf{J}_{\theta}^\top(\mathcal{X},\theta^*) \mathbf{J}_{\theta}(\mathcal{X},\theta^*) \diag(\theta^*)}{\sigma^2} \right)^{-1}
    \\&= \left( \frac{\mathbf{I}_{P}}{\alpha^2} + \frac{\diag(\theta^*)\mathbf{J}_{\theta}^\top \mathbf{J}_{\theta} \diag(\theta^*)}{\sigma^2} \right)^{-1}
\end{align*}
Moreover, the mean of posterior is also invariant to $\mathcal{T}$.
\begin{align*}
    &\frac{\Sigma_{\mathbb{Q}}\diag(\mathcal{T}(\theta^*))\mathbf{J}_{\theta}^{\top}(\mathcal{X},\mathcal{T}(\theta^*))\left( \mathcal{Y} - f(\mathcal{X},\mathcal{T}(\theta^*))\right)}{\sigma^2}
    \\&= \frac{\Sigma_{\mathbb{Q}}\diag(\mathcal{T}(\theta^*))\mathbf{J}_{\theta}^{\top}(\mathcal{X},\mathcal{T}(\theta^*))\left( \mathcal{Y} - f(\mathcal{X},\theta^*)\right)}{\sigma^2}
    \\&= \frac{\Sigma_{\mathbb{Q}}\diag(\theta^*)\mathbf{J}_{\theta}^{\top}(\mathcal{X},\theta^*)\left( \mathcal{Y} - f(\mathcal{X},\theta^*)\right)}{\sigma^2}
    \\&= \frac{\Sigma_{\mathbb{Q}}\diag(\theta^*)\mathbf{J}_{\theta}^{\top}\left( \mathcal{Y} - f(\mathcal{X},\theta^*)\right)}{\sigma^2}
\end{align*}
Therefore, \eqref{eq:ctkgp-c-posterior-mean} and \eqref{eq:ctkgp-c-posterior-cov} are invariant to function-preserving scale transformations. The remain part of the theorem is related to the definition of function-preserving scale transformation $\mathcal{T}$. For generalization error, following holds
\begin{align*}
    \err_{\mathcal{D}}(\mathbb{Q}_{\mathcal{T}(\theta^*)})
    &= \mathbb{E}_{(x,y) \sim \mathcal{D}, \psi \sim \mathbb{Q}_{\mathcal{T}(\theta^*)}} [\err(f(x,\psi), y) ]\\
    &= \mathbb{E}_{(x,y) \sim \mathcal{D}, c \sim \mathbb{Q}_{\mathcal{T}(\theta^*)}} [\err(g^\mathrm{pert}_{\theta^*}(x,c), y) ]\\
    &= \mathbb{E}_{(x,y) \sim \mathcal{D}, c \sim \mathbb{Q}_{\theta^*}} [\err(g^\mathrm{pert}_{\theta^*}(x,c), y) ]\\
    &= \mathbb{E}_{(x,y) \sim \mathcal{D}, \psi \sim \mathbb{Q}_{\theta^*}} [\err(f(x,\psi), y) ]\\
    &= \err_{\mathcal{D}}(\mathbb{Q}_{\theta^*})
\end{align*}
This proof can be extended to the empirical error $\err_{\mathcal{S}_\mathbb{Q}}$.
\end{proof}

\subsection{Proof of Theorem~\ref{thm:pac-bayes-ctk}}

\begin{proof}
\textbf{(Construction of KL divergence)} To construct PAC-Bayes-CTK, we need to arrange KL divergence between posterior and prior as follows:
\begin{align*}
    \mathrm{KL}[\mathbb{\mathbb{Q}}\| \mathbb{P}] 
    &=
    \frac{1}{2}\left(\mathrm{tr}\left(\Sigma_{\mathbb{P}}^{-1}(\Sigma_{\mathbb{Q}}+(\mu_{\mathbb{Q}}-\mu_{\mathbb{P}})(\mu_{\mathbb{Q}}-\mu_{\mathbb{P}})^{\top})\right)
    + \log|\Sigma_{\mathbb{P}}|
    -\log|\Sigma_{\mathbb{Q}}|- P\right) \\
    &= \frac{1}{2}\mathrm{tr}(\Sigma_{\mathbb{P}}^{-1}(\mu_{\mathbb{Q}}-\mu_{\mathbb{P}})(\mu_{\mathbb{Q}}-\mu_{\mathbb{P}})^{\top})) + \frac{1}{2}\left(\mathrm{tr}(\Sigma_{\mathbb{P}}^{-1}\Sigma_{\mathbb{Q}}) + \log|\Sigma_{\mathbb{P}}|
    -\log|\Sigma_{\mathbb{Q}}|- P\right)\\
    &= 
    \frac{1}{2}(\mu_{\mathbb{Q}}-\mu_{\mathbb{P}})^{\top}\Sigma_{\mathbb{P}}^{-1}(\mu_{\mathbb{Q}}-\mu_{\mathbb{P}})
    + \frac{1}{2}\left(\mathrm{tr}(\Sigma_{\mathbb{P}}^{-1}\Sigma_{\mathbb{Q}})
    -\log|\Sigma_{\mathbb{P}}^{-1}\Sigma_{\mathbb{Q}}|- P\right)\\
    &=\underbrace{\frac{\mu_{\mathbb{Q}}^{\top}\mu_{\mathbb{Q}}}{2\alpha^2}}_\textrm{perturbation}
    + \underbrace{\frac{1}{2}\left(\mathrm{tr}(\Sigma_{\mathbb{P}}^{-1}\Sigma_{\mathbb{Q}})-\log|\Sigma_{\mathbb{P}}^{-1}\Sigma_{\mathbb{Q}}|- p\right)}_\textrm{sharpness}
\end{align*}
where the first equality uses the KL divergence between two Gaussian distributions, the thrid equality uses trace property ($\mathrm{tr}(AB) = \mathrm{tr}(BA)$ and $\mathrm{tr}(a)=a$ for scalar $a$), and the last equality uses the definition of PAC-Bayes prior ($\mathbb{P}_{\theta^*}(c)=\mathcal{N}(c|\mathbf{0}_{P}, \alpha^2 \mathbf{I}_{P})$). For sharpness term, we first compute the $\Sigma_\mathbb{P}^{-1}\Sigma_\mathbb{Q}$ term as
\begin{align*}
    \Sigma_\mathbb{P}^{-1}\Sigma_\mathbb{Q} = \left( \mathbf{I}_{P} + \frac{\alpha^2}{\sigma^2}\mathbf{J}_{c}^\top \mathbf{J}_{c} \right)^{-1}
\end{align*}
Since $\alpha^2, \sigma^2 > 0$ and $\mathbf{J}_{c}^\top \mathbf{J}_{c}$ is positive semi-definite, the matrix $\Sigma_\mathbb{P}^{-1}\Sigma_\mathbb{Q}$ have non-zero eigenvalues of $\{ \beta_i \}_{i=1}^{P}$. Since trace is the sum of eigenvalues and log-determinant is the sum of log-eigenvalues, we have
\begin{align*}
    \mathrm{KL}[\mathbb{\mathbb{Q}}\| \mathbb{P}] 
    &= \frac{\mu_{\mathbb{Q}}^{\top}\mu_{\mathbb{Q}}}{2\alpha^2}
    + \frac{1}{2}\sum_{i=1}^{P} \left(\beta_i -\log (\beta_i) - 1 \right)
    \\&= \frac{\mu_{\mathbb{Q}}^{\top}\mu_{\mathbb{Q}}}{2\alpha^2}
    + \frac{1}{2}\sum_{i=1}^{P} h(\beta_i)
\end{align*}
where $h(x) = x - \log(x) -1$. By plugging this KL divergence to the \eqref{eq:pac-bayes-dep}, we get \eqref{eq:pac-bayes-connectivity}.

\textbf{(Eigenvalues of $\Sigma_\mathbb{P}^{-1}\Sigma_\mathbb{Q}$)} To show the scale-invariance of PAC-Bayes-CTK, it is sufficient to show that KL divergence posterior and prior is scale-invariant: $\frac{\log(2\sqrt{N_\mathbb{Q}}/\delta)}{2N_\mathbb{Q}}$ is independent to KL PAC-Bayes prior/posterior and we already show the invariance property of empirical/generalization error term in Proposition \ref{prop:distribution-invariance}. To show the invariance property of KL divergence, we consider the Connectivity Tangent Kernel (CTK) as defined in \eqref{coro:ctk-eigenvalues}:
\begin{align*}
    \mathbf{C}_\mathcal{X}^{\theta^*}:= \mathbf{J}_{c} \mathbf{J}_{c}^\top = \mathbf{J}_{\theta} \diag (\theta^*)^2 \mathbf{J}_{\theta}^\top  \in \mathbb{R}^{N K \times N K}.
\end{align*}
Since CTK is a real-symmetric matrix, one can assume the eigenvalue decomposition of CTK as $\mathbf{C}_\mathcal{X}^{\theta^*} = Q \Lambda Q^\top$ where $Q \in \mathbb{R}^{NK \times NK}$ is an orthogonal matrix and $\Lambda \in \mathbb{R}^{NK \times NK}$ is a diagonal matrix. Then following holds for $\Sigma_\mathbb{P}^{-1}\Sigma_\mathbb{Q}$
\begin{align*}
    \Sigma_\mathbb{P}^{-1}\Sigma_\mathbb{Q}
    &= \left( \mathbf{I}_{P} + \frac{\alpha^2}{\sigma^2}\mathbf{J}_{c}^\top \mathbf{J}_{c} \right)^{-1}
    \\&= \left( \mathbf{I}_{P} + \frac{\alpha^2}{\sigma^2}Q\Lambda Q^\top  \right)^{-1}
    \\&= Q \left(\mathbf{I}_{P} + \frac{\alpha^2}{\sigma^2}\Lambda \right)^{-1} Q^\top
\end{align*}
Therefore, eigenvalues of $\Sigma_\mathbb{P}^{-1}\Sigma_\mathbb{Q}$ are $\frac{1}{1+\alpha^2 \lambda_i/\sigma^2} = \frac{\sigma^2}{\sigma^2 + \alpha^2\lambda_i}$ where $\{\lambda_i\}_{i=1}^{P}$ are eigenvalues of CTK (and diagonal elements of $\Lambda$).

\textbf{(Scale invariance of CTK)} The scale-invariance property of CTK is a simple application of \eqref{eq:jacobian-c-invariance}:
\begin{align*}
    \mathbf{C}^{\mathcal{T}(\psi)}_{xy} 
    &= \mathbf{J}_{\theta}(x,\mathcal{T}(\psi))\mathrm{diag}(\mathcal{T}(\psi)^2) \mathbf{J}_{\theta}(y,\mathcal{T}(\psi))^{\top}
    \\&=\mathbf{J}_{\theta}(x,\psi)\mathcal{T}^{-1}\mathcal{T}\diag(\psi)\diag(\psi)\mathcal{T}\mathcal{T}^{-1}\mathbf{J}_{\theta}(x,\psi)^\top
    \\&=\mathbf{J}_{\theta}(x,\psi)\diag(\psi)\diag(\psi)\mathbf{J}_{\theta}(x,\psi)^\top
    \\&= \mathbf{C}^\psi_{xy} 
    \text{ , }\forall x,y\in \mathbb{R}^{D}, \forall \psi \in \mathbb{R}^{P}.
\end{align*}
Therefore, CTK is invariant to any function-preserving scale transformation $\mathcal{T}$ and so do its eigenvalues. This guarantees the invariance of $\Sigma_\mathbb{P}^{-1}\Sigma_\mathbb{Q}$ and its eigenvalues. In summary, we showed the scale-invariance property of sharpness term of KL divergence. Now all that remains is to show the invariance of the perturbation term. However, this is already proved in the proof of Proposition \ref{prop:distribution-invariance}. Therefore, we show PAC-Bayes-CTK is invariant to any function-preserving scale transformation.
\end{proof}

\subsection{Proof of Corollary \ref{coro:ctk-eigenvalues}}

\begin{proof}
In proof of Theorem \ref{thm:pac-bayes-ctk}, we showed that eigenvalues of $\Sigma_\mathbb{P}^{-1}\Sigma_\mathbb{Q}$ can be represented as 
\begin{align*}
    \frac{\sigma^2}{\sigma^2 + \alpha^2\lambda_i}
\end{align*}
where $\{\lambda_i\}_{i=1}^{P}$ are eigenvalues of CTK. Now, we identify the eigenvalues of CTK. To this end, we assume the singular value decomposition (SVD) of Jacobian w.r.t. connectivity $\mathbf{J}_{c} \in \mathbb{R}^{NK \times P}$ as $\mathbf{J}_{c} = U\Sigma V^\top$ where $U \in \mathbb{R}^{NK \times NK}$ and $V \in \mathbb{R}^{P \times P}$ are orthogonal matrices and $\Sigma \in \mathbb{R}^{NK \times P}$ is a rectangular diagonal matrix. Then, CTK can be represented as $\mathbf{C}_\mathcal{X}^{\theta^*} = \mathbf{J}_{c} \mathbf{J}_{c}^\top = U\Sigma V^\top V \Sigma U^\top = U \Sigma^2 U^\top$. In summary, the column vectors of $U$ are eigenvectors of CTK and eigenvalues of CTK are square of singular values of $\mathbf{J}_c$ and so $\lambda_i\geq 0,  \forall i$. Therefore $\beta_i \le 1$ for all $i=1,\dots,P$ for eigenvalues $\{\beta_i\}_{i=1}^{P}$ of $\Sigma_\mathbb{P}^{-1}\Sigma_\mathbb{Q}$ and equality holds for $\lambda_i=0$. Now all that remains is to show that the sharpness term of PAC-Bayes-CTK is a monotonically increasing function on each eigenvalues of CTK. To show this, we first keep in mind that
\begin{equation*}
    s(x) := \frac{\sigma^2}{\sigma^2 + \alpha^2 x}
\end{equation*}
is a monotonically decreasing function for $x \ge 0$ and $h(x) := x - \log(x) - 1$ is a monotonically decreasing function for $x \in (0,1]$. Since sharpness term of KL divergence is 
\begin{align*}
    \sum_{i=1}^{P}\frac{h(\beta_i)}{4N_\mathbb{Q}}
    = \sum_{i=1}^{P}\frac{(h\circ s)(\lambda_i)}{4N_\mathbb{Q}}
\end{align*}
This is a monotonically increasing function for $x \ge 0$ since $s(x) \le 1$ for $x \ge 0$. For your information, we plot $y= h(x)$ and $y= (h\circ s)(x)$ in Figure \ref{fig:functions}.

\begin{figure*}[!t]
    \begin{subfigure}[t]{0.45\linewidth}
        \includegraphics[width=\linewidth]{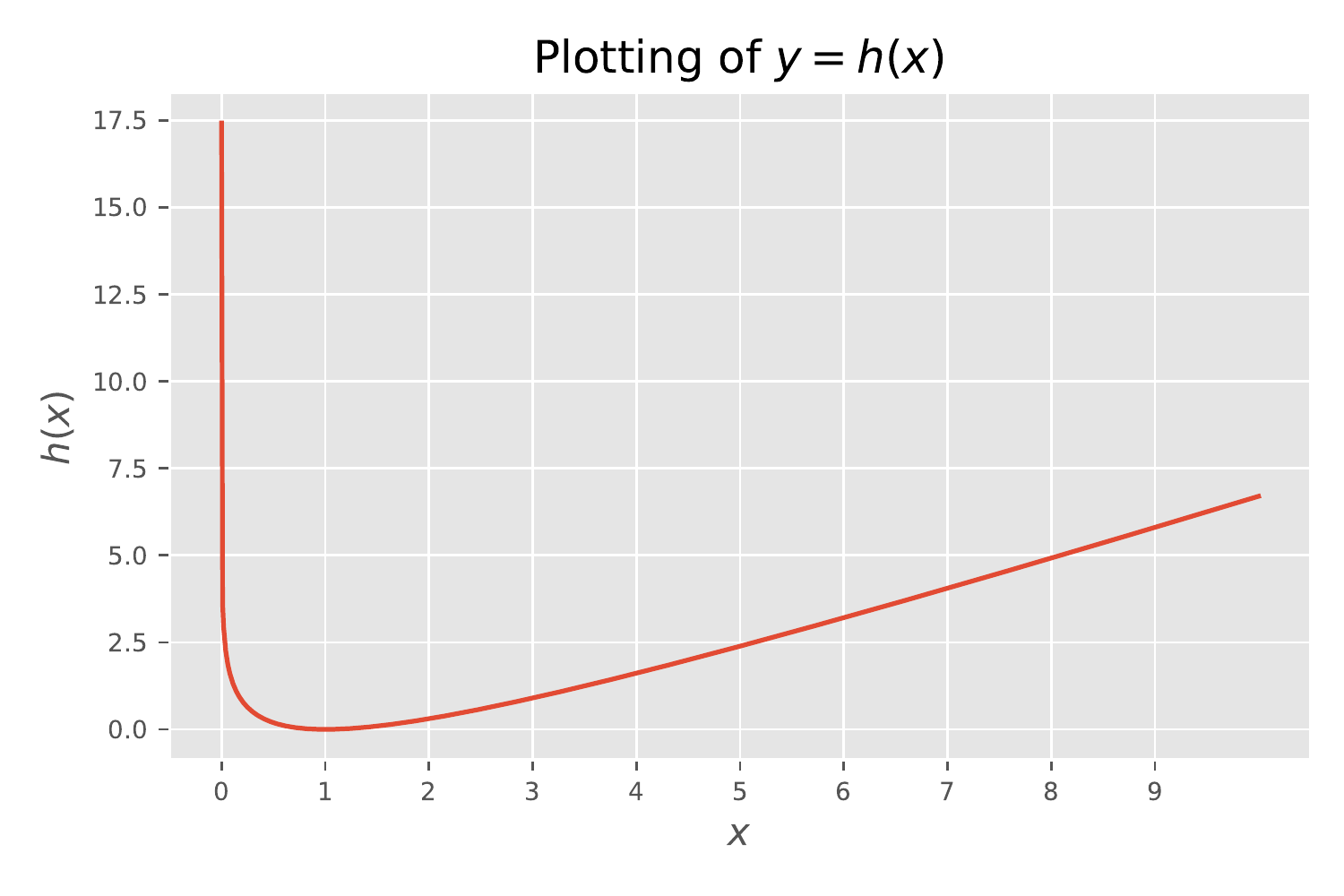}
        \captionsetup{justification=centering}
        \caption{$y= h(x)$}
    \end{subfigure}
    \begin{subfigure}[t]{0.45\linewidth}
        \includegraphics[width=\linewidth]{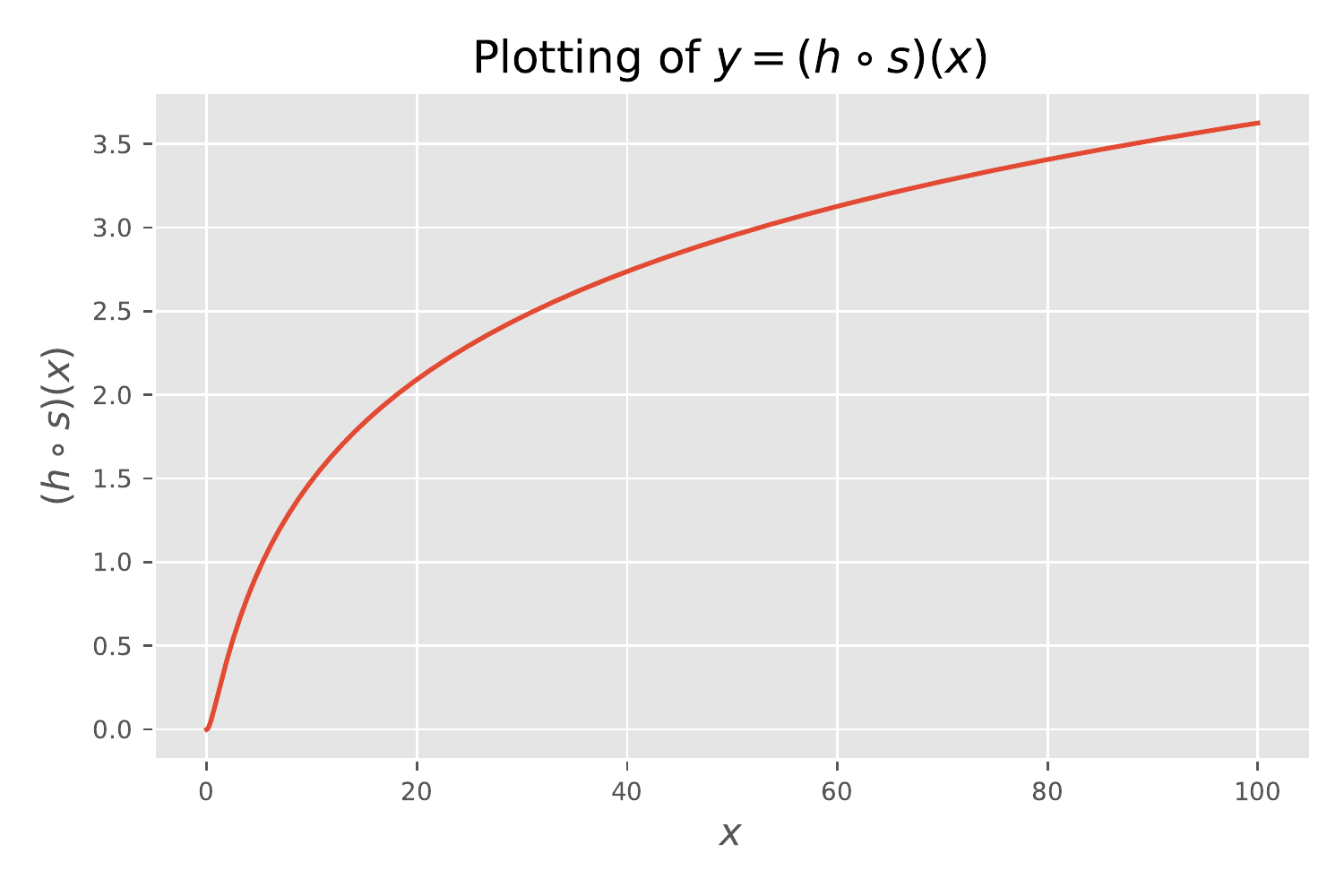}
        \captionsetup{justification=centering}
        \caption{$y= (h\circ s)(x)$ where $\sigma = \alpha = 1$}
    \end{subfigure}
    \caption{Functions used in proofs}
    \label{fig:functions}
\end{figure*}
\end{proof}

\subsection{Proof of Proposition \ref{prop:ctk-invariance}}

We refer to \textbf{Scale invariance of CTK} part of proof of Theorem \ref{thm:pac-bayes-ctk}. This is a direct application of scale-invariance property of Jacobian w.r.t. connectivity.

\subsection{Proof of Corollary \ref{coro:connectivity-sharpness}}

\begin{proof}
Since CS is trace of CTK, it is a sum of eigenvalues of CTK. As shown in the proof of Corollary \ref{coro:ctk-eigenvalues}, eigenvalues of CTK are square of singular values of Jacobian w.r.t. connectivity $c$. Therefore, eigenvalues of CTK are non-negative and vanishes to zero if CS vanishes to zero.
\begin{align*}
    \sum_{i=1}^{P}\lambda_i = 0 \Rightarrow \lambda_i = 0 \Rightarrow \beta_i = s(\lambda_i) = 1
    \Rightarrow h(\beta_i)=0,\quad\forall i =1,\dots,P
\end{align*}
This means the sharpness term of KL divergence vanishes to zero. Furthermore, singular values of Jacobian w.r.t. $c$ also vanishes to zero in this case. Therefore, $\mu_\mathbb{Q}$ vanishes to zero, also. Similarly, if CS diverges to infinity, this means (at least) one of eigenvalues of CTK diverges to infinity. In this case, following holds
\begin{align*}
    \lambda_i \rightarrow \infty \Rightarrow \beta_i = s(\lambda_i) \rightarrow 0
    \Rightarrow h(\beta_i) \rightarrow \infty,\quad\forall i =1,\dots,P
\end{align*}
Therefore, KL divergence term of PAC-Bayes-CTK also diverges to infinity.
\end{proof}

\subsection{Proof of Proposition \ref{prop:uncertainty-amplifying}}

\begin{proof}
By assumption, we fixed all non-scale invariant parameters. This means we exclude these parameters in sampling procedure of CL and LL. In terms of predictive distribution, this can be translated as
\begin{align*}
    f^\mathrm{lin}_{\theta^*}(x,\psi)&|p_\mathrm{LA}(\psi|\mathcal{S})
    \sim \mathcal{N}(f(x,\theta^*), \alpha^2 \hat{\Theta}_{xx}^{\theta^*} - \alpha^2 \hat{\Theta}_{x\mathcal{X}}^{\theta^*}\hat{\Theta}{\mathcal{X}}^{\theta^*-1}\hat{\Theta}_{\mathcal{X}x}^{\theta^*})\\
    f^\mathrm{lin}_{\theta^*}(x,\psi)&|p_\mathrm{CL}(\psi|\mathcal{S})
    \sim \mathcal{N}(f(x,\theta^*), \alpha^2 \hat{\mathbf{C}}_{xx}^{\theta^*} -\alpha^2  \hat{\mathbf{C}}_{x\mathcal{X}}^{\theta^*}\hat{\mathbf{C}}_{\mathcal{X}}^{\theta^*-1}\hat{\mathbf{C}}_{\mathcal{X}x}^{\theta^*})
\end{align*}
where $\hat{\Theta}_{xx'}^\psi := \sum_{i\in \mathcal{P}} \frac{\partial f(x, \psi)}{\partial \theta_i}\frac{\partial f(x', \psi)}{\partial \theta_i}$ and $\hat{\mathbf{C}}_{xx'}^\psi:= \sum_{i\in \mathcal{P}} \frac{\partial f(x, \psi)}{\partial \theta_i}\frac{\partial f(x', \psi)}{\partial \theta_i}(\psi_i)^2$ for scale-invariant parameter set $\mathcal{P}$. Thereby, we mask the gradient of non scale-invariant parameter as zero. Therefore, this can be arrange as follows
\begin{align*}
    \hat{\Theta}_{xx'}^\psi = \mathbf{J}_{\theta}(x,\psi) \diag(\mathbf{1}_\mathcal{P})\mathbf{J}_{\theta}(x,\psi)^\top, \quad
    \hat{\mathbf{C}}_{xx'}^\psi = \mathbf{J}_{\theta}(x,\psi) \diag(\psi)\diag(\mathbf{1}_\mathcal{P})\diag(\psi)\mathbf{J}_{\theta}(x,\psi)^\top
\end{align*}
where $\mathbf{1}_\mathcal{P} \in \mathbb{R}^{P}$ is a masking vector (i.e., one for included components and zero for excluded components).
Then, the weight decay regularization for scale-invariant parameters can be represented as 
\begin{align*}
    \mathcal{W}_\gamma (\psi)_{i} = 
    \begin{cases}
    \gamma \psi_i, &\text{if} \quad \psi_i \in \mathcal{P}.\\
    \psi_i, &\text{if} \quad \psi_i \not\in \mathcal{P}.\\
    \end{cases}
\end{align*}
Therefore, we get
\begin{align*}
    \hat{\Theta}_{xx'}^{\mathcal{W}_{\gamma}(\psi)} 
    &= \mathbf{J}_{\theta}(x,\mathcal{W}_{\gamma}(\psi)) \diag(\mathbf{1}_\mathcal{P})\mathbf{J}_{\theta}(x,\mathcal{W}_{\gamma}(\psi)))^\top
    \\&= \mathbf{J}_{\theta}(x,\psi) \mathcal{W}_{\gamma}^{-1}\diag(\mathbf{1}_\mathcal{P})\mathcal{W}_{\gamma}^{-1}\mathbf{J}_{\theta}(x,\psi)^\top
    \\&= \mathbf{J}_{\theta}(x,\psi) \diag(\mathbf{1}_\mathcal{P}/\gamma^2)\mathbf{J}_{\theta}(x,\psi)^\top
    \\&= 1/\gamma^2 \mathbf{J}_{\theta}(x,\psi) \diag(\mathbf{1}_\mathcal{P})\mathbf{J}_{\theta}(x,\psi)^\top
    \\&= 1/\gamma^2\hat{\Theta}_{xx'}^\psi
\end{align*}
for empirical NTK and 
\begin{align*}
    \hat{\mathbf{C}}_{xx'}^{\mathcal{W}_{\gamma}(\psi)}
    &= \mathbf{J}_{\theta}(x,\mathcal{W}_{\gamma}(\psi)) \diag(\mathcal{W}_{\gamma}(\psi))\diag(\mathbf{1}_\mathcal{P})\diag(\mathcal{W}_{\gamma}(\psi))\mathbf{J}_{\theta}(x,\mathcal{W}_{\gamma}(\psi)))^\top
    \\&= \mathbf{J}_{\theta}(x,\psi) \mathcal{W}_{\gamma}^{-1}\mathcal{W}_{\gamma}\diag(\psi)\diag(\mathbf{1}_\mathcal{P})\diag(\psi)\mathcal{W}_{\gamma}\mathcal{W}_{\gamma}^{-1}\mathbf{J}_{\theta}(x,\psi)^\top
    \\&= \mathbf{J}_{\theta}(x,\psi) \diag(\psi)\diag(\mathbf{1}_\mathcal{P})\diag(\psi)\mathbf{J}_{\theta}(x,\psi)^\top
    \\&=\hat{\mathbf{C}}_{xx'}^\psi
\end{align*}
for empirical CTK. Therefore, we get
\begin{align*}
    f^\mathrm{lin}_{\mathcal{W}_{\gamma}(\theta^*)}(x,\psi)&|p_\mathrm{LA}(\psi|\mathcal{S})
    \sim \mathcal{N}(f(x,\theta^*), \alpha^2 / \gamma^2 \hat{\Theta}_{xx}^{\theta^*} - \alpha^2 / \gamma^2 \hat{\Theta}_{x\mathcal{X}}^{\theta^*}\hat{\Theta}_{\mathcal{X}}^{\theta^*-1}\hat{\Theta}_{\mathcal{X}x}^{\theta^*})\\
    f^\mathrm{lin}_{\mathcal{W}_{\gamma}(\theta^*)}(x,\psi)&|p_\mathrm{CL}(\psi|\mathcal{S})
    \sim \mathcal{N}(f(x,\theta^*), \alpha^2 \hat{\mathbf{C}}_{xx}^{\theta^*} - \alpha^2 \hat{\mathbf{C}}_{x\mathcal{X}}^{\theta^*}\hat{\mathbf{C}}_{\mathcal{X}}^{\theta^*-1}\hat{\mathbf{C}}_{\mathcal{X}x}^{\theta^*})
\end{align*}
This gives us 
\begin{align*}
    &\textrm{Var}_{\psi \sim p_\mathrm{LA}(\psi|\mathcal{S})}(f^\mathrm{lin}_{\mathcal{W}_{\gamma}(\theta^*)}(x,\psi)) = \textrm{Var}_{\psi \sim p_\mathrm{LA}(\psi|\mathcal{S})}(f^\mathrm{lin}_{\theta^*}(x,\psi))/\gamma^2\\
    &\textrm{Var}_{\psi \sim p_\mathrm{CL}(\psi|\mathcal{S})}(f^\mathrm{lin}_{\mathcal{W}_{\gamma}(\theta^*)}(x,\psi)) = \textrm{Var}_{\psi \sim p_\mathrm{CL}(\psi|\mathcal{S})}(f^\mathrm{lin}_{\theta^*}(x,\psi))
\end{align*}
\end{proof}

\subsection{Derivation of PAC-Bayes-NTK}

\begin{theorem}[PAC-Bayes-NTK] \label{thm:pac-bayes-ntk}
    Let us assume pre-trained parameter $\theta^*$ with data $\mathcal{S}_\mathbb{P}$. Let us assume PAC-Bayes prior and posterior as
    \begin{align}
        \mathbb{P'}_{\theta*}(\delta) &:= \mathcal{N}(\delta|\mathbf{0}_{P}, \alpha^2\mathbf{I}_{P})\label{eq:ntkgp-w-prior}\\
        \mathbb{Q'}_{\theta^*}(\delta) &:= \mathcal{N}(\delta| \mu_{\mathbb{Q'}}, \Sigma_{\mathbb{Q'}})\label{eq:ntkgp-w-posterior}\\
        \mu_{\mathbb{Q'}} &:= \frac{\Sigma_{\mathbb{Q'}}\mathbf{J}_{\theta}^\top \left( \mathcal{Y} - f(\mathcal{X},\theta^*)\right)}{\sigma^2}\label{eq:ntkgp-w-posterior-mean}\\
        \Sigma_{\mathbb{Q'}} &:= \left( \frac{\mathbf{I}_{P}}{\alpha^2} + \frac{\mathbf{J}_{
        \theta}^\top \mathbf{J}_{\theta}}{\sigma^2} \right)^{-1}\label{eq:ntkgp-w-posterior-cov}
    \end{align}
    
    By applying $\mathbb{P'}_{\theta^*}, \mathbb{Q'}_{\theta^*}$ to data-dependent PAC-Bayes bound (\eqref{eq:pac-bayes-dep}), we get
    \begin{align}\label{eq:pac-bayes-ntk}
        \err_{\mathcal{D}}(\mathbb{Q'}_{\theta^*}) 
        &\le \err_{\mathcal{S}_\mathbb{Q'}}(\mathbb{Q'}_{\theta^*}) + \sqrt{
        \overbrace{\underbrace{\frac{\mu_{\mathbb{Q'}}^{\top}\mu_{\mathbb{Q'}}}{4\alpha^2 N_\mathbb{Q'}}}_\textrm{(average)   perturbation} + \underbrace{\sum_{i=1}^{P}\frac{h\left(\beta'_i \right)}{{4N_\mathbb{Q'}}}}_\textrm{sharpness}}^\textrm{KL divergence} + \frac{\log(2\sqrt{N_\mathbb{Q'}}/\delta)}{2N_\mathbb{Q'}}}
    \end{align}
    where $\{\beta'_i\}_{i=1}^{P}$ are eigenvalues of $(\mathbf{I}_{P} + \frac{\alpha^2}{\sigma^2}\mathbf{J}_{\theta}^\top \mathbf{J}_{\theta})^{-1}$ and $h(x) := x-\log(x)-1$. This upper bound is not scale-invariant in general. 
\end{theorem}

\begin{proof}
    The main difference between PAC-Bayes-CTK and PAC-Bayes-NTK is the definition of Jacobian: PAC-Bayes-CTK use Jacobian w.r.t connectivity and PAC-Bayes-NTK use Jacobian w.r.t. parameter. Therefore, \textbf{Construction of KL divergence} of proof of Theorem \ref{thm:pac-bayes-ctk} is preserved except 
    \begin{align*}
        \Sigma_{\mathbb{P'}}^{-1}\Sigma_{\mathbb{Q'}}
        = (\mathbf{I}_{P} + \frac{\alpha^2}{\sigma^2}\mathbf{J}_{\theta}^\top \mathbf{J}_{\theta})^{-1}
    \end{align*}
    and $\beta'_i$ are eigenvalues of $\Sigma_{\mathbb{P'}}^{-1}\Sigma_{\mathbb{Q'}}$. Note that these eigenvalues satisfies
    \begin{align*}
        \beta'_i = \frac{\sigma^2}{\sigma^2 + \alpha^2 \lambda'_i}
    \end{align*}
    where $\lambda'_i$ are eigenvalues of NTK.
\end{proof}

\begin{remark}[Function-preserving scale transformation to NTK]
 On the contrary to the CTK, scale invariance property is not applicable to the NTK due to Jacobian w.r.t. parameter:
\begin{align*}
    \mathbf{J}_{\theta}(x, \mathcal{T}(\psi)) = \frac{\partial }{\partial \mathcal{T}(\psi)} f(x, \mathcal{T}(\psi)) = \frac{\partial }{\partial \mathcal{T}(\psi)} f(x, \psi) = \mathbf{J}_{\theta}(x, \psi) \mathcal{T}^{-1}
\end{align*}
If we assume all parameters are scale-invariant (or equivalently masking the Jacobian for all non scale-invariant parameters as in the proof of Proposition \ref{prop:uncertainty-amplifying}), the scale of NTK is proportional to the inverse scale of parameters. 
\end{remark}

\subsection{Deterministic limiting kernel of CTK}\label{proof:deterministic-ctk}

\begin{theorem}[Deterministic limiting kernel of CTK]
Let us assume $L$-layered network with Lipschitz activation function and NN with NTK initialization. Then the empirical CTK converges in probability to a deterministic limiting kernel $\mathbf{C}_{xy}$ as the layers width $n_{1}, \dots, n_{L} \rightarrow \infty$ sequentially. Furthermore, $\mathbf{C}_{xy}= \Theta_{xy}$ holds.
\end{theorem}

\begin{proof}
    The proof is a modification to proof of convergence of NTK in \citet{jacot2018ntk} considering NTK initialization~(i.e. standard Gaussian for all parameters). We provide proof by induction. For single layer network, The CTK is summed as:
    \begin{align*}
        (\mathbf{C}_{xx'})_{kk'} 
        &= \frac{1}{n_0} \sum_{i=1}^{n_0}\sum_{j=1}^{n_1} x_i x_i^{'} \delta_{jk}\delta_{jk'}W_{ik}W_{ik'}
        + \beta^2 \sum_{j=1}^{n_1}\delta_{jk}\delta_{jk'}\\
        &\rightarrow (\Theta_{xx})_{kk'}
    \end{align*}
    since the weight is sampled from standard Gaussian distribution, whose variance is 1, and product of two (independent) random variable converges in probability converges to the product of converged values. If we assume CTK of $l$-th layer is converged to NTK of $l$-th layer in probability, then the convergence of the $(l+1)$-th layer is also  satisfied since multiplication of two random weights, which converges to 1, is multiplied to the empirical NTK of $(l+1)$-th layer, which converges to the deterministic limiting NTK of $(l+1)$-th layer. Therefore, empirical CTK converges in probability to the deterministic limiting CTK, which is equivalent to the deterministic limiting NTK.
\end{proof}

\section{Details of Squared Loss for Classification Tasks}\label{sec:classification}
For the classification tasks in Sec.~\ref{subsec:exp-CTKGP}, we use the squared loss instead of the cross-entropy loss since our theoretical results are built on the squared loss. Here, we describe how we use the squared loss to mimic the cross-entropy loss. There are several works~\citep{lee2020finite, hui2021evaluation} that utilized the squared loss for the classification task instead of the cross-entropy loss. Specifically, \citet{lee2020finite} used 
\begin{align*}
    \mathcal{L}(\mathcal{S},\theta)=\frac{1}{2NK}\sum_{(x_i,y_i)\in\mathcal{S}}\|f(x_i,\theta)-y_i\|^2
\end{align*}
where $C$ is the number of classes, and \citet{hui2021evaluation} used
\begin{align*}
    \ell((x,c), \theta) =\frac{1}{2K}\left(k(f_c(x,\theta)-M)^2+\sum_{i=1,i\neq c}^K f_i(x,\theta)^2\right)
\end{align*}
for single data loss, where $\ell((x,c),\theta)$ is sample loss given input $x$, target $c$ and parameter $\theta$, $f_i(x,\theta)\in\mathbb{R}$ is the $i$-th component of $f(x,\theta)\in\mathbb{R}^K$, $k$ and $M$ are dataset-specific hyper-parameters.

These works used the mean for reducing the vector-valued loss into a scalar loss. However, this can be problematic when the number of classes is large. When the number of classes increases, the denominator of the mean (the number of classes) increases while the target value remains 1~(one-hot label). As a result, the scale of a gradient for the target class becomes smaller. To avoid such an unfavorable effect, we just use the sum for reducing vector-valued loss into a scalar loss instead of taking the mean, i.e.,
\begin{align*}
    \ell((x,c), \theta) = \frac{1}{2} \left( (f_c(x,\theta)-1)^2 + \sum_{i=1, i\neq c}^K f_i(x,\theta)^2 \right)
\end{align*}

This analysis is consistent with the hyper-parameter selection in~\citet{hui2021evaluation}. They used larger $k$ and $M$ as the number of classes increases (e.g., $k=1, M=1$ for CIFAR-10~\citep{krizhevsky2009learning}, but $k=15, M=30$ for ImageNet~\citep{deng2009imagenet}) which results in manually compensating the scale of gradient to the target class label.

\section{Derivation of PAC-Bayes posterior}\label{sec:pac-bayes-posterior-derivation}

\textbf{Derivation of $\mathbb{Q}_{\theta^*}(c)$}

For Bayesian linear regression, we compute the posterior of $\beta \in \mathbb{R}^{P}$
\begin{align*}
    y_i = x_i \beta + \epsilon_i,\quad\text{for }i=1\dots,M
\end{align*}
where $\epsilon_i \sim \mathcal{N}(0, \sigma^2)$ is i.i.d. sampled and the prior of $\beta$ is given as $\beta \sim \mathcal{N}(\mathbf{0}_{P}, \alpha^2 \mathbf{I}_{P})$. By concatenating this, we get
\begin{align*}
    \mathbf{y} = \mathbf{X}\beta + \varepsilon
\end{align*}
where $\mathbf{y} \in \mathbb{R}^{M}, \mathbf{X} \in \mathbb{R}^{M \times p}$ are concatenation of $y_i, x_i$, respectively and $\varepsilon \sim \mathcal{N}(\mathbf{0}_{M}, \sigma^2 \mathbf{I}_{M})$. It is well known \citep{bishop2006pattern, murphy2012machine} that the posterior of $\beta$ for this problem is 
\begin{align*}
    \beta &\sim \mathcal{N}\left(\beta| \mu, \Sigma \right)\\
    \mu &:= \frac{\Sigma \mathbf{X}^\top y}{\sigma^2} \\
    \Sigma &:= \left( \frac{\mathbf{I}_{P}}{\alpha^2} + \frac{\mathbf{X}^\top \mathbf{X}}{\sigma^2} \right)^{-1}.
\end{align*}
Similarly, we define Bayesian linear regression problem as
\begin{align*}
    y_i = f(x_i, \theta^*) + \mathbf{J}_{\theta}(x_i,\theta^*)\diag(\theta^*) c + \epsilon_i,\quad\text{for }i=1\dots,NK
\end{align*}
where $M = NK$ and the regression coefficient is $\beta = c$ in this case. Thus, we treat $y_i - f(x_i,\theta^*)$ as a target and $\mathbf{J}_{\theta}(x_i, \theta^*)\diag({\theta^*})$ as an input of linear regression problem. By concatenating this, we get
\begin{equation*}
    \mathcal{Y} = f(\mathcal{X},\theta^*) + \mathbf{J}_{c}c + \varepsilon
    \Rightarrow
    \left( \mathcal{Y} - f(\mathcal{X},\theta^*)\right) = \mathbf{J}_{c}c + \varepsilon.
\end{equation*}
By plugging this to the posterior of Bayesian linear regression problem, we get
\begin{align*}
\mathbb{Q}_{\theta^*}(c) &:= \mathcal{N}(c| \mu_{\mathbb{Q}}, \Sigma_{\mathbb{Q}})\\
\mu_{\mathbb{Q}} &:= \frac{\Sigma_{\mathbb{Q}}\mathbf{J}_{c}^\top \left( \mathcal{Y} - f(\mathcal{X},\theta^*)\right)}{\sigma^2} = \frac{\Sigma_{\mathbb{Q}}\diag(\theta^*)\mathbf{J}_{\theta}^{\top}\left( \mathcal{Y} - f(\mathcal{X},\theta^*)\right)}{\sigma^2}\\
\Sigma_{\mathbb{Q}} &:= \left( \frac{\mathbf{I}_{P}}{\alpha^2} + \frac{\mathbf{J}_{c}^\top \mathbf{J}_{c}}{\sigma^2} \right)^{-1} = \left( \frac{\mathbf{I}_{P}}{\alpha^2} + \frac{\diag(\theta^*)\mathbf{J}_{\theta}^\top \mathbf{J}_{\theta} \diag(\theta^*)}{\sigma^2} \right)^{-1}
\end{align*}

\noindent\textbf{Derivation of $\mathbb{Q}_{\theta^*}(\psi)$}\\

We define perturbed parameter $\psi$ as follows
\begin{align*}
    \psi := \theta^* + \theta^* \odot c.
\end{align*}
Since $\psi$ is affine to $c$, we get the distribution of $\psi$ as
\begin{align*}
    \mathbb{Q}_{\theta^*}(\psi) &:= \mathcal{N}(\psi|\mu_{\mathbb{Q}}^{\psi}, \Sigma_{\mathbb{Q}}^{\psi})\\
    \mu_{\mathbb{Q}}^{\psi}&:= \theta^* + \theta^* \odot \mu_{\mathbb{Q}}\\
    \Sigma_{\mathbb{Q}}^{\psi}&:= \diag(\theta^*)\Sigma_{\mathbb{Q}}\diag(\theta^*)
    = \left( \frac{\diag(\theta^*)^{-2}}{\alpha^2} + \frac{\mathbf{J}_{\theta}^\top \mathbf{J}_{\theta}}{\sigma^2} \right)^{-1}
\end{align*}

\section{Representative cases of function-preserving scaling transformations}\label{sec:scaling-transformations}

\textbf{Activation-wise rescaling transformation \citep{tsuzuku2020nfm, neyshabur2015path}} For NNs with ReLU activations, following holds for $\forall x \in \mathbb{R}^{d}, \gamma >0$: $f(x, \theta) = f(x, \mathcal{R}_{\gamma,l,k}(\theta))$, where rescale transformation $\mathcal{R}_{\gamma,l,k}(\cdot)$\footnote{For a simple two layer linear NN $f(x):=W_2\sigma(W_1x)$ with weight matrix $W_1, W_2$, the first case of \eqref{eq:rescaling} corresponds to $k$-th row of $W_1$ and the second case of \eqref{eq:rescaling} corresponds to $k$-th column of $W_2$.} is defined as
\begin{align}
    \footnotesize{(\mathcal{R}_{\gamma,l,k}(\theta))_{i} = 
    \begin{cases}
    \gamma \cdot \theta_i&,\text{ if $\theta_i \in$ \{param. subset connecting as input edges to $k$-th activation at $l$-th layer\}}\\
    \theta_i / \gamma &,\text{ if $\theta_i \in$  \{param. subset connecting as output edges to $k$-th activation at $l$-th layer\}}\\\label{eq:rescaling}
    \theta_{i}&,\text{ for $\theta_i$ in the other cases}
    \end{cases}}
\end{align}
Note that $\mathcal{R}_{\gamma,l,k}(\cdot)$ is a finer-grained rescaling transformation than layer-wise rescaling transformation (i.e. common $\gamma$ for all activations in layer $l$) discussed in \citet{dinh2017sharp}. \citet{dinh2017sharp} showed that even layer-wise rescaling transformations can sharpen pre-trained solutions in terms of trace of Hessian (i.e., contradicting the FM hypothesis). This contradiction also occurs to previous PAC-Bayes bounds \citep{tsuzuku2020nfm, kwon2021asam} due to the scale-dependent term.

\textbf{Weight decay with BN layers \citep{ioffe2015batch}}
For parameters $W \in \mathbb{R}^{n_{l} \times n_{l-1}}$ preceding BN layer,
\begin{align}
    \mathrm{BN}((\diag(\gamma)W)u)  = \mathrm{BN}(Wu) \label{eq:bn-scale}
\end{align}
for an input $u\in\mathbb{R}^{n_{l-1}}$ and a positive vector $\gamma \in \mathbb{R}^{n_l}_{+}$. This implies that scaling transformations on these parameters preserve function represented by NNs for $\forall x \in \mathbb{R}^{d}, \gamma \in \mathbb{R}^{n_{l}}_{+}$: $f(x, \theta) = f(x, \mathcal{S}_{\gamma,l,k}(\theta))$, where scaling transformation $\mathcal{S}_{\gamma, l, k}(\cdot)$ is defined for $i=1,\dots,P$
\begin{align}
\left(\mathcal{S}_{\gamma,l,k}(\theta)\right)_{i} = 
    \footnotesize{\begin{cases}
    \gamma_{k} \cdot \theta_{i}&,\text{ if $\theta_i \in$ \{param. subset connecting as input edges to $k$-th activation at $l$-th layer\}}\\
    \theta_{i}&,\text{for $\theta_i$ in the other cases}
    \end{cases}}
\end{align}\label{eq:scale-invariance}
Note that the weight decay regularization \citep{loshchilov2018decoupled, zhang2018three} can be implemented as a realization of $\mathcal{S}_{\gamma, l, k}(\cdot)$~(e.g., $\gamma=0.9995$ for all activations preceding BN layers). Therefore, thanks to Theorem \ref{prop:distribution-invariance} and Theorem \ref{prop:ctk-invariance}, our CTK-based bound is invariant to weight decay regularization applied to parameters before BN layers. We also refer to \citep{van2017l2, lobacheva2021periodic} for optimization perspective of weight decay with BN.

\section{Implementation of Connectivity Laplace}\label{sec:kfac-cl}

To estimate the empirical/generalization bound in Sec. \ref{sec:computing-bound} and calibrate uncertainty in Sec. \ref{subsec:exp-CTKGP}, we need to sample $c$ from the posterior $\mathbb{Q}_{\theta^*}(c)$. For this, we sample perturbations $\delta$ in connectivity space
\begin{align*}
    \delta \sim \mathcal{N}\left(\delta|\mathbf{0}_{P}, \left( \frac{\mathbf{I}_{P}}{\alpha^2} + \frac{\mathbf{J}_{c}^{\top}\mathbf{J}_{c}}{\sigma^2}\right)^{-1} \right)
\end{align*}
so that $c = \mu_{\mathbb{Q}} + \delta$ for \eqref{eq:ctkgp-c-posterior-mean}. To sample this, we provide a novel approach to sample from LA/CL without curvature approximation. To this end, we consider following optimization problem
\begin{align*}
    \arg\min_{c}L(c) := \arg\min_{c}\frac{1}{2N\sigma^2}\|\mathcal{Y}-f(\mathcal{X},\theta^*)-\mathbf{J}_{c}c + \varepsilon \|^2 + \frac{1}{2N\alpha^2}\|c-c_0 \|^2_2\\
\end{align*}
where $\varepsilon \sim \mathcal{N}(\varepsilon | \mathbf{0}_{NK}, \sigma^2\mathbf{I}_{NK})$ and $c_0 \sim \mathcal{N}(c_0 | \mathbf{0}_{P}, \alpha^2\mathbf{I}_{P})$. By first-order optimality condition, we have
\begin{align*}
    N\nabla_{c}L(c) = -\frac{\mathbf{J}_{c}^\top(\mathcal{Y}-f(\mathcal{X},\theta^*)-\mathbf{J}_{c}c^* + \varepsilon)}{\sigma^2} + \frac{c^*-c_0}{\alpha^2}= \mathbf{0}_{P}.
\end{align*}
By arranging this w.r.t. optimizer $c^*$, we get
\begin{align*}
    c^* 
    &= \left(\mathbf{J}_{c}^\top \mathbf{J}_c + \frac{\sigma^2}{\alpha^2} \mathbf{I}_{P}\right)^{-1}\left( \mathbf{J}_{c}^\top (\mathcal{Y}- f(\mathcal{X},\theta^*)) + \frac{\sigma^2}{\alpha^2}c_0 + \mathbf{J}_c \varepsilon \right)
    \\&= \left(\mathbf{J}_{c}^\top \mathbf{J}_c + \frac{\sigma^2}{\alpha^2} \mathbf{I}_{P}\right)^{-1}\mathbf{J}_{c}^\top (\mathcal{Y}- f(\mathcal{X},\theta^*)) + \left(\mathbf{J}_{c}^\top \mathbf{J}_c + \frac{\sigma^2}{\alpha^2} \mathbf{I}_{P}\right)^{-1} \left( \frac{\sigma^2}{\alpha^2}c_0 + \mathbf{J}_c \varepsilon \right)
    \\&=\underbrace{\left(\frac{\mathbf{I}_{P}}{\alpha^2} + \frac{\mathbf{J}_c^\top \mathbf{J}_c}{\sigma^2 }\right)^{-1}\frac{\mathbf{J}_c^\top (\mathcal{Y}-f(\mathcal{X},\theta^*))}{\sigma^2}}_{\textrm{Deterministic}} + \underbrace{\left(\frac{\mathbf{I}_{P}}{\alpha^2} + \frac{\mathbf{J}_c^\top \mathbf{J}_c}{\sigma^2 }\right)^{-1}\left( \frac{c_0}{\alpha^2} + \frac{\mathbf{J}_{c}^\top \varepsilon}{\sigma^2} \right)}_\textrm{Stochastic}
\end{align*}
Since both $\varepsilon$ and $c_0$ are sampled from independent Gaussian distributions, we have
\begin{align*}
    z:= \left( \frac{c_0}{\alpha^2} + \frac{\mathbf{J}_{c}^\top \varepsilon}{\sigma^2} \right)
    \sim \mathcal{N}\left(z|\mathbf{0}_{P}, \frac{\mathbf{I}_P}{\alpha^2} + \frac{\mathbf{J}_c^\top \mathbf{J}_c}{\sigma^2}  \right)
\end{align*}
Therefore, optimal solution of randomized optimization problem $\arg\min_{c} L(c)$ is
\begin{align*}
    c \sim \mathcal{N}\left( c \, \Big|\left( \frac{\mathbf{I}_{P}}{\alpha^2} + \frac{\mathbf{J}_c^\top \mathbf{J}_c}{\sigma^2 }\right)^{-1}\frac{\mathbf{J}_c^\top (\mathcal{Y}-f(\mathcal{X},\theta^*))}{\sigma^2}, \left(\frac{\mathbf{I}_{P}}{\alpha^2} + \frac{\mathbf{J}_c^\top \mathbf{J}_c}{\sigma^2 }\right)^{-1} \right) = \mathcal{N}(c| \mu_\mathbb{Q}, \Sigma_\mathbb{Q})
\end{align*}
Similarly, sampling from CL can be implemented as a following optimization problem.
\begin{align*}
    \arg\min_{c}L(c) := \arg\min_{c}\frac{1}{2N\sigma^2}\|\mathbf{J}_{c}c - \varepsilon \|^2 + \frac{1}{2N\alpha^2}\|c-c_0 \|^2_2\\
\end{align*}
where $\varepsilon \sim \mathcal{N}(\varepsilon | \mathbf{0}_{NK}, \sigma^2\mathbf{I}_{NK})$ and $c_0 \sim \mathcal{N}(c_0 | \mathbf{0}_{P}, \alpha^2\mathbf{I}_{P})$. Since we sample the noise of data/perturbation and optimize the perturbation, this can be interpreted as a Randomize-Then-Optimize implementation of Laplace Approximation and Connectivity Laplace \citep{bardsley2014randomize, matthews2017sample}.

\section{Details of computing Connectivity Sharpness}\label{sec:hutchison}

It is well known that empirical NTK or Jacobian is intractable in modern architecture of NNs (e.g., ResNet \citep{he2016deep} or BERT \citep{devlin2018bert}). Therefore, one might wonder how Connectivity Sharpness can be computed for these architectures. However, Connectivity Sharpness in Sec. \ref{Sec:sub_CS} is defined as trace of empirical CTK, thereby one can compute CS with Hutchison's method \citep{hutchinson1989stochastic, meyer2021hutch++}. According to Hutchison's method, trace of a matrix $A \in \mathbb{R}^{m \times m}$ is 
\begin{align*}
    \mathrm{tr}(A) = \mathrm{tr}(A \mathbf{I}_{p}) = \mathrm{tr}(A\mathbb{E}_{z}[zz^\top]) = \mathbb{E}_{z}[\mathrm{tr}(Azz^\top)] = \mathbb{E}_{z}[\mathrm{tr}(z^\top Az)]
    = \mathbb{E}_{z}[z^\top Az]
\end{align*}
where $z \in \mathbb{R}^{m}$ is a random variable with $\mathrm{cov}(z) = \mathbf{I}_{m}$ (e.g., standard normal distribution or Rademacher distribution). Since $A=\mathbf{C}^{\theta^*}_{\mathcal{X}} = \mathbf{J}_{c}(\mathcal{M}, \mathbf{0}_{p}) \mathbf{J}_{c}(\mathcal{M}, \mathbf{0}_{p})^\top \in \mathbb{R}^{Nk}$ in our case, we further use mini-batch approximation to compute $z^\top A z$: (\romannumeral 1) Sample $z_M \in \mathbb{R}^{Mk}$ from Rademacher distribution for mini-batch $\mathcal{M}$ with size $M$ and (\romannumeral 2) compute $v_\mathcal{M} := \mathbf{J}_{c}(\mathcal{M}, \mathbf{0}_{p})^\top z_\mathcal{M} \in \mathbb{R}^{P}$ with Jacobian-vector product of JAX \citep{jax2018github} and (\romannumeral 3) compute $x_\mathcal{M} = \|v_\mathcal{M}\|^2_2$. Then, the sum of $x_M$ for all mini-batch in training dataset is a Monte-Carlo approximation of CS with sample size 1. Empirically, we found that this approximation is sufficiently stable to capture the correlation between sharpness and generalization as shown in Sec. \ref{subsec:gen-metric}.

\section{Predictive uncertainty of Connectivity/Linearized Laplace}\label{sec:derivation-pred-unc}

In this section, we derive predictive uncertainty of Linearized Laplace (LL) and Connectivity Laplace (CL). By matrix inversion lemma \citep{woodbury1950inverting}, the weight covariance of LL is 
\begin{align*}
    (\mathbf{I}_{p}/\alpha^2 + \mathbf{J}_{\theta}(\mathcal{X}, \theta^*)^\top \mathbf{J}_{\theta}(\mathcal{X}, \theta^*)/\sigma^2)^{-1}
    = 
    \alpha^2 \mathbf{I}_{p} - \alpha^2 \mathbf{J}_{\theta}(\mathcal{X}, \theta^*)^\top (\frac{\sigma^2}{\alpha^2} \mathbf{I}_{Nk} + \Theta_{\mathcal{XX}}^{\theta^*})^{-1}\mathbf{J}_{\theta}(\mathcal{X}, \theta^*).
\end{align*}
Therefore, if $\sigma^2/\alpha^2 \rightarrow 0$, then the weight covariance of LL converges to
\begin{align*}
    \alpha^2 \mathbf{I}_{p} - \alpha^2 \mathbf{J}_{\theta}(\mathcal{X}, \theta^*)^\top \Theta_{\mathcal{XX}}^{\theta^*-1}\mathbf{J}_{\theta}(\mathcal{X}, \theta^*).
\end{align*}
With this weight covariance and linearized NN, the predictive uncertainty of LL is
\begin{align*}
    f^\mathrm{lin}_{\theta^*}(x,\theta)&|p_\mathrm{LA}(\theta|\mathcal{S})
    \sim \mathcal{N}(f(x,\theta^*), \alpha^2 \Theta_{xx}^{\theta^*} - \alpha^2 \Theta_{x\mathcal{X}}^{\theta^*}\Theta_{\mathcal{XX}}^{\theta^*-1}\Theta_{\mathcal{X}x}^{\theta^*}).
\end{align*}
Similarly, the predictive uncertainty of CL is 
\begin{align*}
    f^\mathrm{lin}_{\theta^*}(x,\theta)&|\mathbb{Q}_{\theta^*}(\theta)
    \sim \mathcal{N}(f(x,\theta^*), \alpha^2 \mathbf{C}_{xx}^{\theta^*} -\alpha^2  \mathbf{C}_{x\mathcal{X}}^{\theta^*}\mathbf{C}_{\mathcal{XX}}^{\theta^*-1}\mathbf{C}_{\mathcal{X}x}^{\theta^*}).
\end{align*}

\section{Details on sharpness-generalization experiments}\label{supp:exp-corr}

To verify that the CS has a better correlation with generalization performance compared to existing sharpness measures, we evaluate the three metrics: (a) Kendall's rank-correlation coefficient $\tau$ \citep{kendall1938new} which considers the consistency of a sharpness measure with generalization gap (i.e., if one has higher sharpness, then so has higher generalization gap) (b) granulated Kendall's coefficient \citep{jiang2020fantastic} which examines Kendall's rank-correlation coefficient w.r.t. individual hyper-parameters to separately evaluate the effect of each hyper-parameter to generalization gap (c) conditional independence test \citep{jiang2020fantastic} which captures the causal relationship between measure and generalization.

\begin{table}[h]
    \centering
    \caption{Configuration of hyper-parameter}
    \begin{tabular}{c|c}
    \toprule
        network depth & 1, 2, 3 \\
        network width & 32, 64, 128 \\
        learning rate & 0.1, 0.032, 0.001 \\
        weight decay & 0.0, 1e-4, 5e-4 \\
        mini-batch size & 256, 1024, 4096 \\
    \bottomrule
    \end{tabular}
    \label{table:hyp-config}
\end{table}

Three metrics are compared with the following baselines: trace of Hessian ($\mathrm{tr}(\mathbf{H})$;  \citep{keskar2017large}), trace of Fisher information matrix ($\mathrm{tr}(\mathbf{F})$; \citep{jastrzebski2021catastrophic}), trace of empirical NTK at $\theta^*$ ($\mathrm{tr}(\mathbf{\Theta^{\theta^*}})$), and four PAC-Bayes bound based measures, sharpness-orig (SO), pacbayes-orig (PO), $1/\alpha'$ sharpness mag (SM), and $1/\sigma'$ pacbayes mag (PM), which are eq. (52), (49), (62), (61) in \citet{jiang2020fantastic}. 

For the granulated Kendall's coefficient, we use 5 hyper-parameters : network depth, network width, learning rate, weight decay and mini-batch size, along with 3 options for each hyper-parameters as in Table~\ref{table:hyp-config}. We use the VGG-13~\citep{Simonyan15} as a base model and we adjust the depth and width of each conv block. We add BN layers after the convolution layer for each block. Specifically, the number of convolution layers of each conv block is the depth and the number of channels of convolution layers of the first conv block is the width. For the subsequent conv blocks, we follow the original VGG width multipliers  ($\times 2$, $\times 4$, $\times 8$). An example with depth 1 and width 128 is depicted in Table~\ref{table:net-config}.

\begin{table}[h]
    \centering 
    \caption{Example of network configuration with respect to the depth 1, width 128 in \citep{Simonyan15}-style.}
    \begin{tabular}{|c|}
    \toprule
        ConvNet Configuration\\ \midrule
        input (224 $\times$ 224 RGB image)\\ \midrule
        Conv3-128\\
        BN\\
        ReLU\\ \midrule
        MaxPool \\ \midrule
        Conv3-256\\
        BN\\
        ReLU\\ \midrule
        MaxPool \\ \midrule
        Conv3-512\\
        BN\\
        ReLU\\ \midrule
        MaxPool \\ \midrule
        Conv3-1024\\
        BN\\
        ReLU\\ \midrule
        MaxPool \\ \midrule
        Conv3-1024\\
        BN\\
        ReLU\\ \midrule
        MaxPool \\ \midrule
        FC-4096 \\
        ReLU\\ \midrule
        FC-4096 \\ 
        ReLU\\ \midrule
        FC-1000 \\
    \bottomrule
    \end{tabular}
    \label{table:net-config}
\end{table}

We use SGD optimizer with a momentum 0.9. We train each model for 200 epochs and use cosine learning rate scheduler~\citep{loshchilov2016sgdr} with 30\% of initial epochs as warm-up epochs. The standard data augmentations (padding, random crop, random horizontal flip, and normalization) for CIFAR-10 is used for training data. For the analysis, we only use models with above 99\% training accuracy following \citet{jiang2020fantastic}. As a result, we use 200 out of 243 trained models for our correlation analysis. For every experiment, we use 8 NVIDIA RTX 3090 GPUs.

\section{Details and additional results on BNN experiments}\label{supp:exp-bnn}

\subsection{Experimental Setting}\label{subsec:exp-setting-uncertainty}


\textbf{Uncertainty calibration on image classification task}
We pre-train models for 200 epochs CIFAR-10/100 dataset~\citep{krizhevsky2009learning} with ResNet-18\citep{he2016deep} as mentioned in Section \ref{sec:computing-bound}. We choose ensemble size $M$ as 8 except Deep Ensemble~\citep{lakshminarayanan2017deep} and Batch Ensemble~\citep{wen2020batchensemble}. We use 4 ensemble members for Deep Ensemble and Batch Ensemble due to computational cost. 

For evaluation, we define single member prediction as one-hot representation of network output with label smoothing. We select label smoothing coefficient as 0.01 for CIFAR-10, 0.1 for CIFAR-100. We define ensemble prediction as averaged prediction of single member predictions. For OOD detection, we use variance of prediction in output space, which is competitive to recent OOD detection methods~\citep{ren2019lrforood, van2020duq}. We use 0.01 for $\sigma$ and select best $\alpha$ with cross validation. For every experiment, we use 8 NVIDIA RTX 3090 GPUs.

\clearpage

\section{Additional results on bound estimation}\label{sec:lanczos}

\begin{table}[h]
\centering
\caption{Results for experiments on PAC-Bayes-NTK estimation.}
\label{tab:bound_ntk}
\resizebox{\textwidth}{!}{%
\begin{tabular}{@{}c|cccc|cccc@{}}
\toprule
                & \multicolumn{4}{c|}{CIFAR-10}             & \multicolumn{4}{c}{CIFAR-100}             \\ \midrule
Parameter scale & 0.5      & 1.0      & 2.0      & 4.0      & 0.5      & 1.0      & 2.0      & 4.0      \\ \toprule
$\mathrm{tr}(\Theta^{\theta^*}_{\mathcal{X}})$ & 18746194.0      & 6206303.5       & 3335419.75      & 2623873.25      & 12688970.0      & 3916139.25      & 2819272.5       & 2662497.0       \\
Bias            & 483.86   & 427.0042 & 299.0085 & 197.3149 & 476.9061 & 478.1776 & 440.284  & 329.8767 \\
Sharpness       & 579.6815 & 472.0    & 402.8186 & 369.3761 & 547.2874 & 434.7583 & 398.5075 & 387.3265 \\
KL divergence        & 531.7708 & 449.5021 & 350.9135 & 283.3455 & 512.0967 & 456.4679 & 419.3957 & 358.6016 \\
Test err.    & 0.5617 ± 0.0670 & 0.4566 ± 0.0604 & 0.2824 ± 0.0447 & 0.1530 ± 0.0199 & 0.6210 ± 0.0096 & 0.6003 ± 0.0094 & 0.5499 ± 0.0100 & 0.4666 ± 0.0093 \\ \midrule
PAC-Bayes-NTK       & 0.7985 ± 0.0694 & 0.6730 ± 0.0626 & 0.4718 ± 0.0465 & 0.3186 ± 0.0202 & 0.8530 ± 0.0140 & 0.8162 ± 0.0136 & 0.7602 ± 0.0112 & 0.6617 ± 0.0114 \\ \bottomrule
\end{tabular}
}
\end{table}

\section{Additional results on image classification}\label{subsec:additional-uncertainty}









\begin{table*}[h]
\caption{\footnotesize Uncertainty calibration results on CIFAR-10~\citep{krizhevsky2009learning} for VGG-13~\citep{Simonyan15}.}
\centering
\resizebox{\textwidth}{!}{%
\begin{tabular}{c|cccc}
\toprule
& \multicolumn{4}{c}{CIFAR-10}           \\ 
\midrule
&NLL~($\downarrow$)&ECE~($\downarrow$)&Brier.~($\downarrow$)&AUC~($\uparrow$) \\
\midrule

Deterministic & 0.4086 ± 0.0018	 & 0.0490 ± 0.0003 & 0.1147 ± 0.0005 & -\\
MCDO & 0.3889 ± 0.0049 & 0.0465 ± 0.0009 & 0.1106 ± 0.0015 & 0.7765 ± 0.0221 \\
MCBN & 0.3852 ± 0.0012 & 0.0462 ± 0.0002 & 0.1108 ± 0.0003 & 0.9051 ± 0.0065 \\
Batch Ensemble & 0.3544 ± 0.0036 & 0.0399 ± 0.0009 & 0.1064 ± 0.0012 & 0.9067 ± 0.0030 \\
Deep Ensemble & 0.2243 & 0.0121 & 0.0776 & 0.7706 \\\midrule
Linearized Laplace & 0.3366 ± 0.0013 & 0.0398 ± 0.0004 & 0.1035 ± 0.0003 & 0.8883 ± 0.0017 \\
Connectivity Laplace (Ours) & 0.2674 ± 0.0028 & 0.0234 ± 0.0011 & 0.0946 ± 0.0010 & 0.9002 ± 0.0033\\
\bottomrule 
\end{tabular}
}
\end{table*}

\begin{table*}[!ht]
\caption{\footnotesize Uncertainty calibration results on CIFAR-100 \citep{krizhevsky2009learning} for VGG-13 \citep{Simonyan15}.}
\centering
\resizebox{\textwidth}{!}{%
\begin{tabular}{c|cccc}
\toprule
& \multicolumn{4}{c}{CIFAR-100}           \\ 
\midrule
&NLL~($\downarrow$)&ECE~($\downarrow$)&Brier.~($\downarrow$)&AUC~($\uparrow$) \\
\midrule
Deterministic & 1.8286 ± 0.0066&	0.1544 ± 0.0010&	0.4661 ± 0.0018	 & -\\
MCDO & 1.7439 ± 0.0089&	0.1363 ± 0.0008&	0.4456 ± 0.0017&	0.6424 ± 0.0099 \\
MCBN & 1.7491 ± 0.0075	&0.1399 ± 0.0010	&0.4488 ± 0.0015	&0.7039 ± 0.0197 \\
Batch Ensemble & 1.6142 ± 0.0101	&0.1077 ± 0.0020	&0.4143 ± 0.0027	&0.7232 ± 0.0021 \\
Deep Ensemble & 1.2006	&0.0456	&0.3228	&0.6929 \\\midrule
Linearized Laplace & 1.5806 ± 0.0054	&0.1036 ± 0.0004	&0.4127 ± 0.0010	&0.6893 ± 0.0221 \\
Connectivity Laplace (Ours) & 1.4073 ± 0.0039	&0.0703 ± 0.0028	&0.3827 ± 0.0012	&0.7254 ± 0.0136\\
\bottomrule 
\end{tabular}
}
\end{table*}

 

\end{document}